\documentclass[letterpaper]{article} 
\usepackage{nips13submit_e,times}

\usepackage{url}
\usepackage{amsmath}
\usepackage{graphicx}
\usepackage[sort&compress,numbers]{natbib}
\usepackage[colorlinks,linkcolor=red,citecolor=blue,urlcolor=blue,draft]{hyperref}
\usepackage{natbibspacing}
\usepackage{algorithm}
\usepackage{algorithmic}
\usepackage{wrapfig}
\usepackage{subcaption}

\usepackage{etoolbox}
\usepackage{tabularx}
\usepackage{mdwlist}

\usepackage{tikz}
\usetikzlibrary{patterns,shadows}
\usepackage{pgfplots}
\usepgfplotslibrary{external}
\definecolor{darkgreen}{rgb}{0.125,0.5,0.169}
\def\gridwidth{2}
\def\nodenum{4}

\def\asnmargin{0.05}

\input{./Definitions}

\begin{document}

\title{Distributed Stochastic Optimization of the Regularized Risk}

\author{ 
  \alignauthor 
  Shin Matsushima\titlenote{Part of the research was performed when
    visiting Purdue University and UCSC.}\\
  \affaddr{University of Tokyo}\\
  \affaddr{Tokyo, Japan}\\
  \email{shin\_matsushima@mist.i.u-tokyo.ac.jp}
  \alignauthor
  Hyokun Yun\\
  \affaddr{Amazon Inc.}\\
  \affaddr{Seattle, WA}\\
  \email{yunhyoku@amazon.com}
  \alignauthor
  Xinhua Zhang\\
  \affaddr{NICTA}\\
  \affaddr{Canberra, Australia}\\
  \email{xinhua.zhang@nicta.com.au}
  \alignauthor
  S.V.\,N. Vishwanathan\\
  \affaddr{University of California}\\
  \affaddr{Santa Cruz, CA}\\
  \email{vishy@ucsc.edu}
}
\date{\today}

\maketitle
\nipsfinalcopy

\begin{abstract}
  Many machine learning algorithms minimize a regularized risk, and
  stochastic optimization is widely used for this task. When working
  with massive data, it is desirable to perform stochastic optimization
  in parallel. Unfortunately, many existing stochastic optimization
  algorithms cannot be parallelized efficiently. In this paper we show
  that one can rewrite the regularized risk minimization problem as an
  equivalent saddle-point problem, and propose an efficient distributed
  stochastic optimization (DSO) algorithm.  We prove the algorithm's
  rate of convergence; remarkably, our analysis shows that the algorithm
  scales almost linearly with the number of processors.  We also verify
  with empirical evaluations that the proposed algorithm is competitive
  with other parallel, general purpose stochastic and batch optimization
  algorithms for regularized risk minimization.
\end{abstract}

\section{Introduction}
\label{sec:Introduction}

Regularized risk minimization is a well-known paradigm in machine
learning:
\begin{align}
  \label{eq:rrm}
  \min_{\wb}
  P\rbr{\wb} := \lambda \sum_{j} \phi_{j}\rbr{w_{j}} + \frac{1}{m}
  \sum_{i=1}^{m} \ell \rbr{\inner{\wb}{\xb_{i}}, y_{i}}.
\end{align}
Here, we are given $m$ training data points $\xb_{i} \in \RR^{d}$ and
their corresponding labels $y_{i}$, while $\wb \in \RR^{d}$ is the
parameter of the model. Furthermore, $w_{j}$ denotes the $j$-th
component of $\wb$, while $\phi_{j}\rbr{\cdot}$ is a convex function
which penalizes complex models. $\ell\rbr{\cdot, \cdot}$ is a loss
function, which is convex in $\wb$. Moreover, $\inner{\cdot}{\cdot}$
denotes the Euclidean inner product, and $\lambda > 0$ is a scalar which
trades-off between the average loss and the regularizer. For brevity, we
will use $\ell_{i}\rbr{\inner{\wb}{\xb_{i}}}$ to denote
$\ell\rbr{\inner{\wb}{\xb_{i}}, y_{i}}$.

Many well-known models can be derived by specializing
\eqref{eq:rrm}. For instance, if $y_{i} \in \cbr{\pm 1}$, then 
setting $\phi_{j}(w_{j}) = \half w_{j}^{2}$ and
$\ell_{i}\rbr{\inner{\wb}{\xb_{i}}} = \max\rbr{0, 1 - y_{i}
  \inner{\wb}{\xb_{i}}}$ recovers binary linear support vector machines
(SVMs) \citep{SchSmo02}. On the other hand, using the same regularizer
but changing the loss function to $\ell_{i}\rbr{\inner{\wb}{\xb_{i}}} =
\log \rbr{1 + \exp \rbr{-y_{i} \inner{\wb}{\xb_{i}}}}$ yields
regularized logistic regression \citep{HasTibFri09}. Similarly,
setting $\phi_{j}\rbr{w_{j}}=\abr{w_{j}}$ leads to sparse learning
such as LASSO \citep{HasTibFri09} with
$\ell_{i}\rbr{\inner{\wb}{\xb_{i}}}=\frac{1}{2} \rbr{y_i
  -\inner{\wb}{\xb_{i}}}^{2}$. 

A number of specialized as well as general purpose algorithms have been
proposed for minimizing the regularized risk. For instance, if both the
loss and the regularizer are smooth, as is the case with logistic
regression, then quasi-Newton algorithms such as L-BFGS \citep{LiuNoc89}
have been found to be very successful. On the other hand, for smooth
regularizers but non-smooth loss functions, \citet{TeoVisSmoLe10}
proposed a bundle method for regularized risk minimization
(BMRM). Another popular first-order solver is Alternating Direction
Method of Multipliers (ADMM) \cite{BoyParChuPelEtal10}. These optimizers
belong to the broad class of batch minimization algorithms; that is, in
order to perform a parameter update, at every iteration they compute the
regularized risk $P(\wb)$ as well as its gradient
\begin{align}
  \label{eq:primal_gradient}
  \nabla P\rbr{\wb} = \lambda \sum_{j=1}^{d} \nabla \phi_{j}\rbr{w_{j}}
  \cdot \eb_{j} + \frac{1}{m} \sum_{i=1}^{m} \nabla \ell_{i}
  \rbr{\inner{\wb}{\xb_{i}}} \cdot \xb_{i},
\end{align}
where $\eb_{j}$ denotes the $j$-th standard basis vector.  Both
$P\rbr{\wb}$ as well as the gradient $\nabla P\rbr{\wb}$ take $O(md)$
time to compute, which is computationally expensive when $m$, the
number of data points, is large. Batch algorithms can be efficiently
parallelized, however, by exploiting the fact that the empirical risk
$\frac{1}{m} \sum_{i=1}^{m} \ell_{i} \rbr{\inner{\wb}{\xb_{i}}}$ as
well as its gradient $\frac{1}{m} \sum_{i=1}^{m} \nabla \ell_{i}
\rbr{\inner{\wb}{\xb_{i}}} \cdot \xb_{i}$ decompose over the data
points, and therefore one can compute $P\rbr{\wb}$ and $\nabla
P\rbr{\wb}$ in a distributed fashion~\cite{ChuKimLinYuetal06}.

Batch algorithms, unfortunately, are known to be unfavorable for large
scale machine learning both empirically and theoretically
\citep{BotBou11}. It is now widely accepted that stochastic algorithms
which process one data point at a time are more effective for
regularized risk minimization. In a nutshell, the idea here is that
\eqref{eq:primal_gradient} can be stochastically approximated by
\begin{align}
  \label{eq:stochastic_gradient}
  \gb_{i} = \lambda \sum_{j=1}^{d} \nabla \phi_{j}\rbr{w_{j}}
  \cdot \eb_{j} + \nabla \ell_{i} \rbr{\inner{\wb}{\xb_{i}}} \cdot
  \xb_{i},
\end{align}
when $i$ is chosen uniformly random in $\cbr{1,\ldots,m}$. Note that
$\gb_{i}$ is an unbiased estimator of the true gradient $\nabla
P\rbr{\wb}$; that is, $\EE_{i\in\cbr{1,\ldots, m}} \sbr{\gb_{i}} =
\nabla P\rbr{\wb}$.  Now we can replace the true gradient by this
\emph{stochastic} gradient to approximate a gradient descent update as
\begin{align}
  \label{eq:sgd-update}
  \wb \leftarrow \wb - \eta \cdot \gb_{i},
\end{align}
where $\eta$ is a step size parameter.  Computing $\gb_i$ only takes
$O(d)$ effort, which is independent of $m$, the number of data
points. \citet{BotBou11} show that stochastic optimization is
asymptotically faster than gradient descent and other second-order
batch methods such as L-BFGS for regularized risk minimization.

However, a drawback of update \eqref{eq:sgd-update} is that it is not
easy to parallelize anymore.  Usually, the computation of $\gb_{i}$ in
\eqref{eq:stochastic_gradient} is a very lightweight operation for
which parallel speed-up can rarely be expected.  On the other hand,
one cannot execute multiple updates of \eqref{eq:sgd-update}
simultaneously, since computing $\gb_{i}$ requires \emph{reading} the
latest value of $\wb$, while updating \eqref{eq:sgd-update} requires
\emph{writing} to the components of $\wb$.  The problem is even more
severe in distributed memory systems, where the cost of communication
between processors is significant.

Existing parallel stochastic optimization algorithms try to work
around these difficulties in a somewhat ad-hoc manner (see
Section~\ref{sec:RelatedWork}). In this paper, we take a fundamentally
different approach and propose a reformulation of the regularized risk
\eqref{eq:rrm}, for which one can \emph{naturally} derive a parallel
stochastic optimization algorithm. Our technical contributions are:
\begin{itemize}
\item We reformulate regularized risk minimization as an equivalent
  saddle-point problem, and show that it can be solved via a new
  distributed stochastic optimization (DSO) algorithm.
\item We prove $O\rbr{1/\sqrt{T}}$ rates of convergence for DSO, and
  show that it scales almost linearly with the number of processors.
\item We verify with empirical evaluations that when used for training
  linear support vector machines (SVMs) or binary logistic regression
  models, DSO is comparable to general-purpose stochastic (\eg,
  \citet{ZinSmoWeiLi10}) or batch (\eg, \citet{TeoVisSmoLe10})
  optimizers.
\end{itemize}


\section{Reformulating Regularized Risk Minimization}
\label{sec:ReformRegulRisk}

We begin by reformulating the regularized risk minimization problem as
an equivalent saddle-point problem.  Towards this end, we first rewrite
\eqref{eq:rrm} by introducing an auxiliary variable $u_i$ for each
$\xb_{i}$:
\begin{align}
  \label{eq:with-u}
  \min_{\wb, \ub} \; \; \lambda \sum_{j=1}^{d} \phi_{j}\rbr{w_{j}} +
  \frac{1}{m} \sum_{i=1}^{m} \ell_{i}\rbr{u_i} 
  \mathrm{ s.t. } \; \; u_i=\inner{\wb}{\xb_i} \quad \forall \, \, i=1,\ldots,m.
\end{align}  
By introducing Lagrange multipliers $\alpha_i$ to eliminate the constraints,
we obtain
\begin{align*}
  \min_{\wb, \ub} \max_{\alphab} \lambda \sum_{j=1}^{d}
  \phi_{j}\rbr{w_{j}} + \frac{1}{m} \sum_{i=1}^{m} \ell_{i}\rbr{u_i} +
  \frac{1}{m}\sum_{i=1}^m \alpha_i( u_i -\inner{\wb}{\xb_i}).
\end{align*}
Here $\ub$ denotes a vector whose components are $u_{i}$. Likewise,
$\alphab$ is a vector whose components are $\alpha_{i}$. Since the
objective function \eqref{eq:with-u} is convex and the constraints are
linear, strong duality applies \citep{BoyVan04}. Therefore, we can
switch the maximization over $\alphab$ and the minimization over $\wb, \ub$.
Note that $\min_{u_i} \alpha_i u_i + \ell_{i}(u_i)$ can be written
$-\ell_{i}^{\star} (-\alpha_{i})$, where $\ell_i^{\star}(\cdot)$ is
the Fenchel-Legendre conjugate of $\ell_{i}(\cdot)$ \citep{BoyVan04}.
The above transformations yield to our formulation:
\begin{align*}
  \max_{\alphab} \min_{\wb} f\rbr{\wb, \alphab} := \lambda
  \sum_{j=1}^{d} \phi_{j}\rbr{w_{j}} - \frac{1}{m} \sum_{i=1}^{m}
  \alpha_{i} \inner{\wb}{\xb_{i}} -\frac{1}{m} \sum_{i=1}^{m}
  \ell^{\star}_{i}\rbr{-\alpha_{i}}.
\end{align*}
If we analytically minimize $f(\wb, \alphab)$ in terms of $\wb$ to
eliminate it, then we obtain so-called \emph{dual} objective which is
only a function of $\alphab$. Moreover, any $\wb^{*}$ which is a
solution of the primal problem \eqref{eq:rrm}, and any $\alphab^{*}$
which is a solution of the dual problem, is a saddle point of
$f\rbr{\wb, \alphab}$ \citep{BoyVan04}. In other words, minimizing the
primal, maximizing the dual, and finding a saddle point of $f\rbr{\wb,
  \alphab}$ are all equivalent problems.




\subsection{Stochastic Optimization}
\label{sec:StochOptim}

Let $x_{ij}$ denote the $j$-th coordinate of $\xb_{i}$, and
$\Omega_{i} := \cbr{j: x_{ij} \neq 0}$ denote the non-zero coordinates
of $\xb_{i}$. Similarly, let $\Omegabar_{j} := \cbr{i: x_{ij} \neq 0}$
denote the set of data points where the $j$-th coordinate is non-zero
and $\Omega := \cbr{\rbr{i,j}: x_{ij} \neq 0}$ denotes the set of all
non-zero coordinates in the training dataset $\xb_{1}, \ldots,
\xb_{m}$. Then, $f(\wb, \alphab)$ can be rewritten as
\begin{align}
  f\rbr{\wb, \alphab} = \sum_{\rbr{i,j} \in \Omega}
    \frac{\lambda \phi_j\rbr{w_j}}{\abr{\Omegabar_j}}
  - \frac{\ell^{\star}_i(-\alpha_i)}{m\abr{\Omega_i}}
  - \frac{\alpha_i w_jx_{ij}}{m}
  = \sum_{\rbr{i,j} \in \Omega} f_{i,j}\rbr{w_{j}, \alpha_{i}},
  \label{eq:sparse_form}
\end{align}
where $|\cdot|$ denotes the cardinality of a set.  Remarkably, each
component $f_{i,j}$ in the above summation depends only on one component
$w_{j}$ of $\wb$ and one component $\alpha_{i}$ of $\alphab$. This
allows us to derive an optimization algorithm which is stochastic in
terms of both $i$ and $j$.
Let us define
\begin{align}
  \gb_{i,j} := \Bigg( \abr{\Omega} \rbr{ \frac{\lambda \nabla\phi_{j}
        \rbr{w_j}}{\abr{\Omegabar_{j}}} - \frac{\alpha_{i} x_{ij}}{m} }
    \eb_{j},
      \abr{\Omega} \rbr{\frac{\nabla\ell^{\star}_{i}(-\alpha_{i})}{ m
        \abr{\Omega_{i}}} - \frac{w_{j} x_{ij}}{m} } \eb_{i} \Bigg).
\end{align}
Under the uniform distribution over $\rbr{i, j} \in \Omega$, one can
easily see that $\gb_{i,j}$ is an unbiased estimate of the gradient of
$f\rbr{\wb, \alphab}$, that is, $\EE_{\cbr{\rbr{i,j} \in \Omega}}
\sbr{\gb_{i,j}} = \rbr{\nabla_{\wb} f\rbr{\wb, \alphab},
  -\nabla_{\alphab} {-f}\rbr{\wb, \alphab}}$.  Since we are
interested in finding a saddle point of $f\rbr{\wb, \alphab}$, our
stochastic optimization algorithm uses the stochastic gradient
$\gb_{i,j}$ to take a \emph{descent} step in $\wb$ and an
\emph{ascent} step in $\alphab$ \citep{NemJudLanSha09}:
\begin{align}
  w_j \leftarrow w_j - \eta \cdot     \rbr{
    \frac{\lambda  \nabla\phi_j\rbr{w_j}}{\abr{\Omegabar_j}} 
    - \frac{\alpha_i x_{ij}}{m}
  },   
  \text{  and  }
  \alpha_i  \leftarrow \alpha_i + \eta \cdot \rbr{
    \frac{ \nabla\ell^{\star}_i(-\alpha_i)}{m\abr{\Omega_i}}
    - \frac{w_j x_{ij}}{m}
  }. 
  \label{eq:saddle_grad}
\end{align}
Surprisingly, the time complexity of update \eqref{eq:saddle_grad} is
independent of the size of data; it is $O(1)$.  Compare this with the
$O(md)$ complexity of batch update and $O(d)$ complexity of regular
stochastic gradient descent.

Note that in the above discussion, we implicitly assumed that
$\phi_{j}\rbr{\cdot}$ and $\ell_{i}^{\star}\rbr{\cdot}$ are
differentiable. If that is not the case, then their derivatives can be
replaced by sub-gradients \citep{BoyVan04}. Therefore this approach can
deal with wide range of regularized risk minimization problem.

\section{Parallelization}
\label{sec:Parallelization}

The minimax formulation \eqref{eq:sparse_form} not only admits an
efficient stochastic optimization algorithm, but also allows us to
derive a distributed stochastic optimization (DSO) algorithm. The key
observation underlying DSO is the following: Given $\rbr{i,j} $ and
$\rbr{i', j'}$ both in $\Omega$, if $i \neq i'$ and $j \neq j'$ then one
can simultaneously perform update \eqref{eq:saddle_grad} on
$(w_{j}, \alpha_{i})$ and $(w_{j'}, \alpha_{i'})$. In other words, the
updates to $w_{j}$ and $\alpha_{i}$ are independent of the updates to
$w_{j'}$ and $\alpha_{i'}$, as long as $i \neq i'$ and $j \neq j'$.

Before we formally describe DSO we would like to present some intuition
using Figure~\ref{fig:dsgd_scheme}.  Here we assume that we have 4
processors. The data matrix $X$ is an $m \times d$ matrix formed by
stacking $\xb_{i}^{\top}$ for $i=1,\ldots,m$, while $\wb$ and $\alphab$
denote the parameters to be optimized. The non-zero entries of $X$ are
marked by an $\mathrm{x}$ in the figure. Initially, both parameters as
well as rows of the data matrix are partitioned across processors as
depicted in Figure~\ref{fig:dsgd_scheme} (left); colors in the figure
denote ownership \eg, the first processor owns a fraction of the data
matrix and a fraction of the parameters $\alphab$ and $\wb$ (denoted as
$\wb^{\rbr{1}}$ and $\alphab^{\rbr{1}}$) shaded with red.  Each
processor samples a non-zero entry $x_{ij}$ of $X$ within the dark
shaded rectangular region (active area) depicted in the figure, and
updates the corresponding $w_{j}$ and $\alpha_{i}$. After performing
updates, the processors stop and exchange coordinates of $\wb$. This
defines an \emph{inner iteration}. After each inner iteration, ownership
of the $\wb$ variables and hence the active area change, as shown in
Figure~\ref{fig:dsgd_scheme} (right). If there are $p$ processors, then
$p$ inner iterations define an \emph{epoch}. Each coordinate of $\wb$ is
updated by each processor at least once in an epoch. The algorithm
iterates over epochs until convergence.

Four points are worth noting. First, since the active area of each
processor does not share either row or column coordinates with the
active area of other processors, as per our key observation above, the
updates can be carried out by each processor in parallel without any
need for intermediate communication with other processors.  Second, we
partition and distribute the data only once. The coordinates of
$\alphab$ are partitioned at the beginning and are not exchanged by the
processors; only coordinates of $\wb$ are exchanged. This means that the
cost of communication is independent of $m$, the number of data
points. Third, our algorithm can work in both shared memory, distributed
memory, and hybrid (multiple threads on multiple machines)
architectures.  Fourth, the $\wb$ parameter is distributed across
multiple machines and there is no redundant storage, which makes the
algorithm scale linearly in terms of space complexity.  Compare this
with the fact that most parallel optimization algorithms require each
local machine to hold a copy of $\wb$.

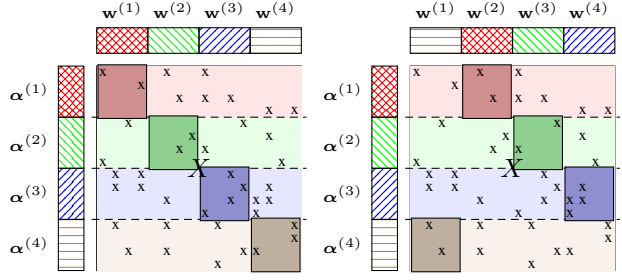
\begin{figure}
  \centering
  \begin{tikzpicture}[scale=0.34]
    
  \draw (0,0) rectangle (\gridwidth * \nodenum, \gridwidth * \nodenum);
  
    \fill [red!10] (0, 3 * \gridwidth) 
    rectangle (\gridwidth * \nodenum, 4 * \gridwidth);

    \fill [green!10] (0, 2 * \gridwidth) 
    rectangle (\gridwidth * \nodenum, 3 * \gridwidth);

    \fill [blue!10] (0, 1 * \gridwidth) 
    rectangle (\gridwidth * \nodenum, 2 * \gridwidth);

    \fill [brown!10] (0,0) 
    rectangle (\gridwidth * \nodenum, 1 * \gridwidth);

  \foreach \y in {1,2,3}
  {
    \draw[densely dashed] 
    (-\gridwidth * 0.1, \gridwidth * \y)
    --
    (\gridwidth * \nodenum + \gridwidth * 0.1,  \gridwidth * \y);
  }

    \drawrowpart
    
    \draw [above] (\gridwidth * 0.5, 4.75 * \gridwidth)
    node {\tiny{$\wb^{(1)}$}};
    \draw [above] (\gridwidth * 1.5, 4.75 * \gridwidth)
    node {\tiny{$\wb^{(2)}$}};
    \draw [above] (\gridwidth * 2.5, 4.75 * \gridwidth)
    node {\tiny{$\wb^{(3)}$}};
    \draw [above] (\gridwidth * 3.5, 4.75 * \gridwidth)
    node {\tiny{$\wb^{(4)}$}};

    \drawcolpart{0}{3}{\ptrna}{\ptcla}
    \drawcolpart{4}{7}{\ptrnb}{\ptclb}
    \drawcolpart{8}{11}{\ptrnc}{\ptclc}
    \drawcolpart{12}{15}{\ptrnd}{\ptcld}

    \fill [red!25!lightgray]
    (0 * \gridwidth + \asnmargin, 
    \gridwidth * 4 + \asnmargin)
    rectangle
    (1 * \gridwidth - \asnmargin, 
    \gridwidth * 3 - \asnmargin);    
    \draw
    (0 * \gridwidth + \asnmargin, 
    \gridwidth * 4 + \asnmargin)
    rectangle
    (1 * \gridwidth - \asnmargin, 
    \gridwidth * 3 - \asnmargin);

    \fill [green!25!lightgray]
    (1 * \gridwidth + \asnmargin, 
    \gridwidth * 3 + \asnmargin)
    rectangle
    (2 * \gridwidth - \asnmargin, 
    \gridwidth * 2 - \asnmargin);    
    \draw
    (1 * \gridwidth + \asnmargin, 
    \gridwidth * 3 + \asnmargin)
    rectangle
    (2 * \gridwidth - \asnmargin, 
    \gridwidth * 2 - \asnmargin);
    
    \fill [blue!25!lightgray]
    (2 * \gridwidth + \asnmargin, 
    \gridwidth * 2 + \asnmargin)
    rectangle
    (3 * \gridwidth - \asnmargin, 
    \gridwidth * 1 - \asnmargin);
    \draw
    (2 * \gridwidth + \asnmargin, 
    \gridwidth * 2 + \asnmargin)
    rectangle
    (3 * \gridwidth - \asnmargin, 
    \gridwidth * 1 - \asnmargin);
    
    \fill [brown!25!lightgray]
    (3 * \gridwidth + \asnmargin, 
    \gridwidth * 1 + \asnmargin)
    rectangle
    (4 * \gridwidth - \asnmargin, 
    \gridwidth * 0 - \asnmargin);
    \draw
    (3 * \gridwidth + \asnmargin, 
    \gridwidth * 1 + \asnmargin)
    rectangle
    (4 * \gridwidth - \asnmargin, 
    \gridwidth * 0 - \asnmargin);
    
    \dtpt{5}{1}  \dtpt{0}{15}  \dtpt{13}{12}  \dtpt{13}{6}  \dtpt{14}{8}  \dtpt{0}{8}  \dtpt{7}{10}  \dtpt{15}{2}  \dtpt{10}{5}  \dtpt{8}{13}  \dtpt{5}{3}  \dtpt{12}{5}  \dtpt{10}{7}  \dtpt{11}{6}  \dtpt{15}{3}  \dtpt{6}{9}  \dtpt{13}{5}  \dtpt{1}{3}  \dtpt{8}{15}  \dtpt{1}{6}  \dtpt{1}{7}  \dtpt{8}{9}  \dtpt{12}{4}  \dtpt{12}{1}  \dtpt{13}{10}  \dtpt{9}{0}  \dtpt{8}{4}  \dtpt{15}{12}  \dtpt{3}{6}  \dtpt{10}{13} 
    \dtpt{6}{13} 
    \dtpt{2}{11}
    \dtpt{5}{5}
    \dtpt{2}{1}
    \dtpt{3}{7}
    \dtpt{9}{3}
    \dtpt{11}{1}
    \dtpt{3}{14}
    \dtpt{11}{11}
    \dtpt{5}{15}

    \draw
    (2 * \gridwidth, 
    \gridwidth * 2)
    node {$X$};
  \end{tikzpicture}
  \begin{tikzpicture}[scale=0.34]

  \draw (0,0) rectangle (\gridwidth * \nodenum, \gridwidth * \nodenum);
  
    \fill [red!10] (0, 3 * \gridwidth) 
    rectangle (\gridwidth * \nodenum, 4 * \gridwidth);

    \fill [green!10] (0, 2 * \gridwidth) 
    rectangle (\gridwidth * \nodenum, 3 * \gridwidth);

    \fill [blue!10] (0, 1 * \gridwidth) 
    rectangle (\gridwidth * \nodenum, 2 * \gridwidth);

    \fill [brown!10] (0,0) 
    rectangle (\gridwidth * \nodenum, 1 * \gridwidth);

  \foreach \y in {1,2,3}
  {
    \draw[densely dashed] 
    (-\gridwidth * 0.1, \gridwidth * \y)
    --
    (\gridwidth * \nodenum + \gridwidth * 0.1,  \gridwidth * \y);
  }
  
    \drawrowpart

    \draw [above] (\gridwidth * 0.5, 4.75 * \gridwidth)
    node {\tiny{$\wb^{(1)}$}};
    \draw [above] (\gridwidth * 1.5, 4.75 * \gridwidth)
    node {\tiny{$\wb^{(2)}$}};
    \draw [above] (\gridwidth * 2.5, 4.75 * \gridwidth)
    node {\tiny{$\wb^{(3)}$}};
    \draw [above] (\gridwidth * 3.5, 4.75 * \gridwidth)
    node {\tiny{$\wb^{(4)}$}};

    \drawcolpart{4}{7}{\ptrna}{\ptcla}
    \drawcolpart{8}{11}{\ptrnb}{\ptclb}
    \drawcolpart{12}{15}{\ptrnc}{\ptclc}
    \drawcolpart{0}{3}{\ptrnd}{\ptcld}

    \fill [brown!25!lightgray]
    (0 * \gridwidth + \asnmargin, 
    \gridwidth * 1 + \asnmargin)
    rectangle
    (1 * \gridwidth - \asnmargin, 
    \gridwidth * 0 - \asnmargin);
    \draw
    (0 * \gridwidth + \asnmargin, 
    \gridwidth * 1 + \asnmargin)
    rectangle
    (1 * \gridwidth - \asnmargin, 
    \gridwidth * 0 - \asnmargin);

    \fill [red!25!lightgray]
    (1 * \gridwidth + \asnmargin, 
    \gridwidth * 4 + \asnmargin)
    rectangle
    (2 * \gridwidth - \asnmargin, 
    \gridwidth * 3 - \asnmargin);    
    \draw
    (1 * \gridwidth + \asnmargin, 
    \gridwidth * 4 + \asnmargin)
    rectangle
    (2 * \gridwidth - \asnmargin, 
    \gridwidth * 3 - \asnmargin);
    
    \fill [green!25!lightgray]
    (2 * \gridwidth + \asnmargin, 
    \gridwidth * 3 + \asnmargin)
    rectangle
    (3 * \gridwidth - \asnmargin, 
    \gridwidth * 2 - \asnmargin);
    \draw
    (2 * \gridwidth + \asnmargin, 
    \gridwidth * 3 + \asnmargin)
    rectangle
    (3 * \gridwidth - \asnmargin, 
    \gridwidth * 2 - \asnmargin);
    
    \fill [blue!25!lightgray]
    (3 * \gridwidth + \asnmargin, 
    \gridwidth * 2 + \asnmargin)
    rectangle
    (4 * \gridwidth - \asnmargin, 
    \gridwidth * 1 - \asnmargin);
    \draw
    (3 * \gridwidth + \asnmargin, 
    \gridwidth * 2 + \asnmargin)
    rectangle
    (4 * \gridwidth - \asnmargin, 
    \gridwidth * 1 - \asnmargin);

    \dtpt{5}{1}  \dtpt{0}{15}  \dtpt{13}{12}  \dtpt{13}{6}  \dtpt{14}{8}  \dtpt{0}{8}  \dtpt{7}{10}  \dtpt{15}{2}  \dtpt{10}{5}  \dtpt{8}{13}  \dtpt{5}{3}  \dtpt{12}{5}  \dtpt{10}{7}  \dtpt{11}{6}  \dtpt{15}{3}  \dtpt{6}{9}  \dtpt{13}{5}  \dtpt{1}{3}  \dtpt{8}{15}  \dtpt{1}{6}  \dtpt{1}{7}  \dtpt{8}{9}  \dtpt{12}{4}  \dtpt{12}{1}  \dtpt{13}{10}  \dtpt{9}{0}  \dtpt{8}{4}  \dtpt{15}{12}  \dtpt{3}{6}  \dtpt{10}{13} 
    \dtpt{6}{13} 
    \dtpt{2}{11}
    \dtpt{5}{5}
    \dtpt{2}{1}
    \dtpt{3}{7}
    \dtpt{9}{3}
    \dtpt{11}{1}
    \dtpt{3}{14}
    \dtpt{11}{11}
    \dtpt{5}{15}

    \draw
    (2 * \gridwidth, 
    \gridwidth * 2)
    node {$X$};

  \end{tikzpicture}
  \caption{Illustration of DSO with 4 processors. The rows of the data
    matrix $X$ as well as the parameters $\wb$ and $\alphab$ are
    partitioned as shown. Colors denote ownerships. The active area of
    each processor is in dark colors. Left: the initial state. Right:
    the state after one bulk synchronization. }
  \label{fig:dsgd_scheme}
\end{figure}

To formally describe DSO, suppose $p$ processors are available, and
let $I_{1}, \ldots, I_{p}$ denote a fixed partition of the set
$\cbr{1, \ldots, m}$ and $J_{1}, \ldots, J_{p}$ denote a fixed
partition of the set $\cbr{1, \ldots, d}$ such that $\abr{I_{q}}
\approx \abr{I_{q'}}$ and $\abr{J_{r}} \approx \abr{J_{r'}}$ for any
$1 \leq q,q',r,r' \leq p$.  We partition the data $\cbr{\xb_{1},
  \ldots, \xb_{m}}$ and the labels $\cbr{y_{1}, \ldots, y_{m}}$ into
$p$ disjoint subsets according to $I_{1}, \ldots, I_{p}$ and
distribute them to $p$ processors.  The parameters $\cbr{\alpha_{1},
  \ldots, \alpha_{m}}$ are partitioned into $p$ disjoint subsets
$\alphab^{(1)}, \ldots, \alphab^{(p)}$ according to $I_{1}, \ldots,
I_{p}$ while $\cbr{w_{1}, \ldots, w_{d}}$ are partitioned into $p$
disjoint subsets $\wb^{(1)}, \ldots, \wb^{(p)}$ according to $J_{1},
\ldots, J_{p}$ and distributed to $p$ processors, respectively. The
partitioning of $\cbr{1,\ldots, m}$ and $\cbr{1, \ldots, d}$ induces a
$p \times p$ partition of $\Omega$:
\begin{align*}
  \Omega^{(q,r)} := \cbr{ (i,j) \in \Omega \;:\; i \in I_q, j \in J_r },
  \;\; q,r \in \cbr{1, \ldots, p}.
\end{align*}
The execution of DSO proceeds in epochs, and each epoch consists of $p$
inner iterations; at the beginning of the $r$-th inner iteration
($r \geq 1$), processor $q$ owns $\wb^{\rbr{\sigma_{r}\rbr{q}}}$ where
$\sigma_{r}\rbr{q} = \cbr{\rbr{q + r - 2} \text{ mod } p} + 1$, and
executes stochastic updates \eqref{eq:saddle_grad} on coordinates in
$\Omega^{\rbr{q,\sigma_{r}\rbr{q}}}$. Since these updates only involve
variables in $\alphab^{\rbr{q}}$ and $\wb^{\rbr{\sigma\rbr{q}}}$, no
communication between processors is required to execute them.  After
every processor has finished its updates,
$\wb^{\rbr{\sigma_{r}\rbr{q}}}$ is sent to machine
$\sigma_{r+1}^{-1}\rbr{\sigma_{r}\rbr{q}}$ and the algorithm moves on to
the $(r+1)$-st inner iteration. Detailed pseudo-code for the DSO
algorithm can be found in Algorithm~\ref{alg:synchronous}.

\renewcommand{\algorithmicrequire}{\textbf{Initialization:}}
\renewcommand{\algorithmicensure}{\textbf{Parallel Iteration:}}

\begin{algorithm}
  \begin{algorithmic}[1]
    \STATE {Each processor $q \in \cbr{1, 2, \ldots, p}$ initializes $\wb^{(q)}$, $\alphab^{(q)}$}

    \STATE $t \leftarrow 1$

    \REPEAT 

    \STATE {$\eta_t \leftarrow \eta_0/\sqrt{t}$}

    \FORALL {$r \in \cbr{1,2,\ldots,p}$}

    \FORALL {\textbf{processors } $q \in \cbr{1,2,\ldots,p}$ \textbf{ in parallel}}

    \FOR {$(i,j) \in \Omega^{(q,\sigma_r(q))}$} 
    \STATE $w_{j} \leftarrow w_{j} - \eta_{t} \cdot \rbr{ \frac{\lambda
         \nabla\phi_j\rbr{w_j}}{\abr{\Omegabar_j}} - \frac{\alpha_{i}x_{ij}}{m}}$
    and $\alpha_{i} \leftarrow \alpha_{i} + \eta_{t} \cdot \rbr{
      \frac{ \nabla\ell^{\star}_i(-\alpha_{i})}{m\abr{\Omega_{i}}} -
      \frac{w_{j} x_{ij}}{m}}$
    
    \ENDFOR

    \STATE {send $\wb^{\rbr{\sigma_r(q)}}$ to machine
      $\sigma_{r+1}^{-1}(\sigma_r(q))$ and receive $\wb^{\rbr{\sigma_{r+1}(q)}}$}
    
    \ENDFOR
    
    \ENDFOR

    \STATE $t \leftarrow t + 1$

    \UNTIL {convergence}
  \end{algorithmic}
  \caption{Distributed stochastic optimization (DSO) for finding saddle
    point of \eqref{eq:sparse_form}}
  \label{alg:synchronous}
\end{algorithm}

\subsection{Convergence Analysis}
\label{sec:ConvergenceProof}
It is known that the stochastic procedure in
section~\ref{sec:StochOptim} is guaranteed to converge to the saddle
point of $f(\wb,\alphab)$ \cite{NemJudLanSha09}.
The main technical difficulty in proving convergence in our case is
because DSO does not sample $\rbr{i,j}$ coordinates uniformly at random
due to its distributed nature. Therefore, first we prove that DSO is
serializable in a certain sense, that is, there exists an ordering of
the updates such that \emph{replaying} them on a single machine would
recover the same solution produced by DSO. We then analyze this serial
algorithm to establish convergence. We believe that this proof technique
is of independent interest, and differs significantly from convergence
analysis for other parallel stochastic algorithms which typically assume
correlation between data points
\citep[e.g.][]{LanSmoZin09,BraKyrBicGue11}.
\newcommand{\maxOmegai}{{\max_i\abs{\Omega_i}}}
\begin{theorem}
  \label{thm:convergence}
  Let $\rbr{\wb^{t}, \alphab^{t}}$ and
  $\rbr{\tilde{\wb}^{t}, \tilde{\alphab}^{t}} :=
  \rbr{\frac{1}{t}\sum_{s=1}^{t} \wb^{s}, \frac{1}{t} \sum_{s=1}^{t}
    \alphab^{s}}$
  denote the parameter values, and the averaged parameter values
  respectively after the $t$-th epoch of
  Algorithm~\ref{alg:synchronous}.  Moreover, assume that
  $\nbr{\wb},\nbr{\alphab},\abr{\nabla\phi_j(w_j)},\abr{\nabla\ell_i^\star\rbr{-\alpha_{i}}},$
  and $\lambda$ are upper bounded by a constant $c>1$. Then, there exists
  a constant $C$, which is dependent only on $c$, such that after $T$
  epochs the duality gap is
  \begin{align}
    \label{eq:thm-rate}
    \varepsilon\rbr{\tilde{\wb}^{T}, \tilde{\alphab}^{T}} := \max_{\alphab'} f\rbr{\tilde{\wb}^{T},
    \alphab'} - \min_{\wb'} f\rbr{\wb', \tilde{\alphab}^{T}} \le
    C \frac{\sqrt{d}}{\sqrt{T}}.
  \end{align}
  On the other hand, if $\phi_j(s) = \half s^2$, $\sqrt{\maxOmegai} < m$
  and $\eta_t < \frac{1}{\lambda}$ hold, then there exists a different
  constant $C'$ dependent only on $c$ and satisfying
  \begin{align}
    \label{eq:thm-special-rate}
    \max_{\alphab'} f\rbr{\tilde{\wb}^{T},\alphab'}
    - \min_{\wb'} f\rbr{\wb', \tilde{\alphab}^{T}} \le \frac{C'}{\sqrt{T}}.
  \end{align}
\end{theorem}
\textbf{Proof:} Please see Appendix~\ref{sec:Proofs}.

To understand the implications of the above theorem, let us assume that
Algorithm~\ref{alg:synchronous} is run with $p \leq \min\rbr{m, d}$
processors with a partitioning of $\Omega$ such that
$\abr{\Omega^{\rbr{q, \sigma_{r}\rbr{q}}}} \approx \frac{1}{p^{2}}
\abr{\Omega}$
and $\abr{J_q} \approx \frac{d}{p}$ for all $q$. As we already noted,
performing updates \eqref{eq:saddle_grad} takes constant time; let us
denote this by $T_{\rm u}$. Moreover, let us assume that communicating
$\wb$ across the network takes constant amount of time denoted by
$T_{\rm c}$, and communicating a subset of $\wb$ takes time proportional
to its cardinality\footnote{Processors communicate on a ring; each
  processor receives $1/p$ fraction of parameters from a predecessor on
  the ring and sends $1/p$ fraction of parameters to a successor on the
  ring. Moreover, as $p$ increases, the size of the messages exchanged
  by the processors decreases. Therefore, our assumption that
  $T_{\rm c}$ is a constant independent of $p$ is reasonable.}. Under
these assumptions, the time for each inner iteration of
Algorithm~\ref{alg:synchronous} can be written as
\begin{align*}
\abs{\Omega^{\rbr{q,\sigma_{r}\rbr{q}}}} T_{\rm u} +
\frac{\abs{J_{\sigma_{r}\rbr{q}}}}{d} T_{\rm c} \approx
\frac{\abs{\Omega}T_{\rm u} }{p^2} + \frac{T_{\rm c}}{p}. 
\end{align*}
Since there are $p$ inner iterations per epoch, the time required to
finish an epoch is ${\abs{\Omega}T_{\rm u} }/ {p} + T_{\rm c}$. As per
Theorem~\ref{thm:convergence} the number of epochs to obtain an
$\epsilon$ accurate solution is independent of $p$.  Therefore, one can
conclude that DSO scales linearly in $p$ as long as
$ {\abr{\Omega}T_{\rm u}}/{p} \gg T_{\rm c}$ holds.  As is to be
expected, for large enough $p$ the cost of communication $T_{\rm c}$
will eventually dominate.

\section{Related Work}
\label{sec:RelatedWork}

Effective parallelization of stochastic optimization for regularized
risk minimization has received significant research attention in recent
years. Because of space limitations, our review of related work will
unfortunately only be partial.

The key difficulty in parallelizing update \eqref{eq:sgd-update} is that
gradient calculation requires us to \emph{read}, while updating the
parameter requires us to \emph{write} to the coordinates of $\wb$.
Consequently, updates have to be executed in serial.  Existing work has
focused on working around the limitation of stochastic optimization by
either a) introducing strategies for computing the stochastic gradient
in parallel (\eg, \citet{LanSmoZin09}), b) updating the parameter in
parallel (\eg, \citet{RecReWriNiu11,BraKyrBicGue11}), c) performing
independent updates and combining the resulting parameter vectors (\eg,
\citet{ZinSmoWeiLi10}), or d) periodically exchanging information
between processors (\eg, \citet{BerTsi97a}). While the former two
strategies are popular in the shared memory setting, the latter two are
popular in the distributed memory setting.  In many cases the
convergence bounds depend on the amount of correlation between data
points and are limited to the case of strongly convex regularizer (\citet{Yan13,ZhaXia15,HsiYuDhi15}).
In contrast our bounds in Theorem~\ref{thm:convergence} do not
depend on such properties of data and more general.
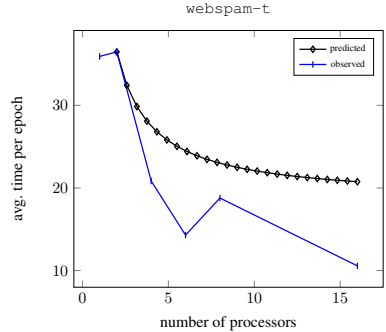
\begin{wrapfigure}{r}{0.5\textwidth}
  \vspace{-0.2 in}
  \begin{center}
  \begin{tikzpicture}[scale=0.6]
    \begin{axis}[xlabel=number of processors,
      ylabel={avg.\ time per epoch},
      legend style={font=\tiny,},
      enlargelimits=true, 
      legend pos=north east,
      title={\texttt{webspam-t}}]
      \addplot[mark=diamond,color=black, thick,domain=2:16] {(35.9269/x) + (36.4735 - (35.9269/2.0))};
      \addlegendentry{predicted}
      \addplot[mark=|,color=blue, thick] table[header=true, x=p, y=webspam] {scale.txt};
      \addlegendentry{observed}
    \end{axis} 
  \end{tikzpicture}
  \end{center}
  \caption{The average time per epoch using $p$ machines on the
    \texttt{webspam-t} dataset. }  
  \label{fig:scale}
  \vspace{-0.3 in}
\end{wrapfigure}

Algorithms that use so-called parameter server to synchronize variable
updates across processors have recently become popular (\eg,
\citet{LiAndSmoYu14}). The main
drawback of these methods is that it is not easy to ``serialize'' the
updates, that is, to replay the updates on a single machine. This
makes proving convergence guarantees, and debugging such frameworks
rather difficult, although some recent progress has been
made~\cite{LiAndSmoYu14}. 

The observation that updates on individual coordinates of the
parameters can be carried out in parallel has been used for other
models. In the context of Latent Dirichlet Allocation,
\citet{YanXuQi09} used a similar observation to derive an efficient
GPU based collapsed Gibbs sampler. On the other hand, for matrix
factorization \citet{GemNijHaaSis11} and \citet{RecRe13} independently
proposed parallel algorithms based on a similar idea. However, to the
best of our knowledge, rewriting \eqref{eq:rrm} as a saddle point
problem in order to discover parallelism is our novel contribution.

\section{Experiments}
\label{sec:EmpiricalEvaluation}

\subsection{Scaling}
\label{sec:Scaling}

We first verify, that the per epoch complexity of DSO scales as
$\abr{\Omega}T_{\rm u} / p + T_{\rm c}$, as predicted by our analysis in
Section~\ref{sec:ConvergenceProof}. Towards this end, we took tthe
\texttt{webspam-t} dataset of \citet{WebCavPu06}, which is one of the
largest datasets we could comfortably fit on a single machine. We let
$p = \cbr{1, 2, 4, 8, 16}$ while fixing the number of cores on each
machine to be 4.

Using the average time per epoch on one and two machines, one can
estimate $\abr{\Omega} T_{\rm u}$ and $T_{\rm c}$. Given these values,
one can then predict the time per iteration for other values of
$p$. Figure~\ref{fig:scale} shows the predicted time and the measured
time averaged over 40 epochs. As can be seen, the time per epoch indeed
goes down as $\approx 1/p$ as predicted by the theory. The test error
and objective function values on multiple machines was very close to the
test error and objective function values observed on a single machine,
thus confirming Theorem~\ref{thm:convergence}.


\subsection{Comparison With Other Solvers}
\label{sec:Comparisonwithother}

\newcommand{\plotserial}[6]{
  \begin{figure}
    \centering
    \tikzsetnextfilename{serial_#1_#2_1e-5_#4_#5}
    \begin{tikzpicture}[scale=0.5]
      \begin{axis}[xlabel=number of iteration,
        ylabel={test rror},
        legend style={font=\tiny,},
        enlargelimits=false, 
        legend pos=north east,
        ymax=0.2,
        xmax=40,
        title={$\lambda=10^{-5}$}]
        \addplot[mark=|,color=blue,  thick] table[header=true, x=iter, y=#5] {../../Results/serial/#2/#1/ada_1e-5.txt};
        \addlegendentry{SGD}
        \addplot[mark=|,color=black, thick] table[header=true, x=iter, y=#6] {../../Results/serial/#2/#1/bmrm_1e-5.txt};
        \addlegendentry{BMRM}
        \addplot[mark=|,color=red, thick] table[header=true, x=iter, y=test_error] {../../Results/serial/#2/#1/saddle_adasvrg_1e-5.txt};
        \addlegendentry{DSO}
      \end{axis} 
    \end{tikzpicture}
    \tikzsetnextfilename{serial_#1_#2_1e-6_#4_#5}
    \begin{tikzpicture}[scale=0.5]
      \begin{axis}[xlabel=number of iteration,
          ylabel={test error},
          legend style={font=\tiny,},
          enlargelimits=false, 
          legend pos=north east,
          ymax=0.2,
          xmax=40,
          restrict x to domain=0:#4, 
          title={$\lambda=10^{-6}$}]
          \addplot[mark=|,color=blue,  thick] table[header=true, x=iter, y=#5] {../../Results/serial/#2/#1/ada_1e-6.txt};
          \addlegendentry{SGD}
          \addplot[mark=|,color=black, thick] table[header=true, x=iter, y=#6] {../../Results/serial/#2/#1/bmrm_1e-6.txt};
          \addlegendentry{BMRM}
          \addplot[mark=|,color=red, thick] table[header=true, x=iter, y=test_error] {../../Results/serial/#2/#1/saddle_adasvrg_1e-6.txt};
          \addlegendentry{DSO}
      \end{axis} 
    \end{tikzpicture}
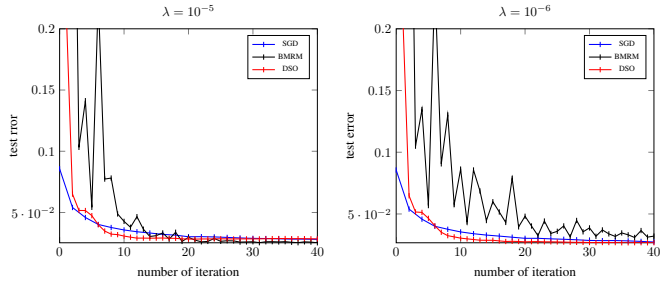
    \caption{The test error of different optimization algorithms on
      linear SVM with \texttt{real-sim} dataset, as a function of the
      number of iteration.}
  \label{fig:serial_expt}
  \end{figure}
  
}

\plotserial{real-sim}{svm}{1e-6}{100}{testerror}{error}

\def\testplot#1#2{
  \def\Task{#1}
  \def\Dataset{#2}
  \paralleltimetestplot
}

\newcommand{\paralleltimetestplot}[9]{
  \tikzsetnextfilename{nips15_parallel_\Dataset_#2_\Task_4}
  \begin{tikzpicture}[scale=0.5]
    \begin{axis}[xlabel=time (seconds),
      ylabel={test error}, 
      legend style={font=\tiny,},
      enlargelimits=false, 
      legend pos=north east,
      restrict y to domain=0:0.5,
      ymin={#8}, 
      ymax={#9}, 
      xmin={0}, 
      xmax={#5},
      title={$\lambda=$ #3}]
      \addplot[mark=o,color=black, dashed, thick]
      table[header=false, x index=3, y index=4]  {../../Results/parallel/\Task/\Dataset/bmrm_#1.txt};
      \addlegendentry{BMRM} 
      \addplot[mark=star,color=blue,thick]
      table[header=true, x=time,y=testerror] {../../Results/parallel/\Task/\Dataset/psgd_adagrad_\Dataset_lambda_#1.txt};
      \addlegendentry{PSGD}
      \addplot[mark=+,color=red,thick]
      table[header=false, x index=2,y index=7, col sep=comma] {../../Results/parallel/\Task/\Dataset/ddso_dcdinit_\Dataset_lambda_#1.txt};
      \addlegendentry{DSO}          
    \end{axis}
  \end{tikzpicture}
}

In our single machine experiments we compare DSO with Stochastic
Gradient Descent (SGD) and Bundle Methods for Regularized risk
Minimization (BMRM) of \citet{TeoVisSmoLe10}. In our multi-machine
experiments we compare with Parallel Stochastic Gradient Descent (PSGD)
of \citet{ZinSmoWeiLi10} and BMRM.  We chose these competitors because,
just like DSO, they are general purpose solvers for regularized risk
minimization \eqref{eq:rrm}, and hence can solve non-smooth problems
such as SVMs as well as smooth problems such as logistic
regression. Moreover, BMRM is a specialized solver for
regularized risk minimization, which has similar performance to other
first-order solvers such as ADMM.
\begin{wrapfigure}{r}{0.5\textwidth}
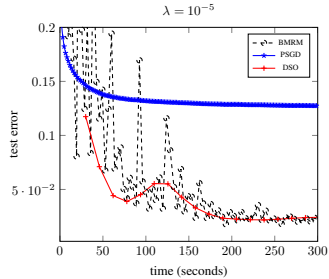

  \begin{center}
    \testplot{logistic}{webspamtrigram}{1e-5}{0.000010}{$10^{-5}$}{50}{300}{0.005}{0.5}{0.002}{0.2}
  \end{center}
  \caption{Logistic regression with \texttt{webspam-t} dataset. Test
    error as a function of elapsed time.}
  \label{fig:par_logistic_compare}
\vspace{-0.2in}
\end{wrapfigure}

We selected two representative datasets and two values of the
regularization parameter $\lambda = \cbr{10^{-5}, 10^{-6}}$ to present
our results. For the single machine experiments we used the
\texttt{real-sim} dataset from \citet{HsiChaLinKeeetal08}, while for the
multi-machine experiments we used \texttt{webspam-t}.  Details of the
datasets can be found in Table~\ref{tab:data} in the appendix.  We use
test error rate as comparison metric, since stochastic optimization
algorithms are efficient in terms of minimizing generalization error,
not training error \citep{BotBou11}.  The results for single machine
experiments on linear SVM training can be found in
Figure~\ref{fig:serial_expt}.  As can be seen, DSO shows comparable
efficiency to that of SGD, and outperforms BMRM.  This demonstrates that
saddlepoint optimization is a viable strategy even in serial setting.

Our multi-machine experimental results for linear SVM training can be
found in Figure~\ref{fig:par_svm_compare}.  As can be seen, PSGD
converges very quickly, but the quality of the final solution is poor;
this is probably because PSGD only solves processor-local problems and
does not have a guarantee to converge to the global optimum.  On the
other hand, both BMRM and DSO converges to similar quality solutions,
and at fairly comparable rates. Similar trends we observed on logistic
regression. Therefore we only show the results with $10^{-5}$ in
Figure~\ref{fig:par_logistic_compare}.
\begin{figure}
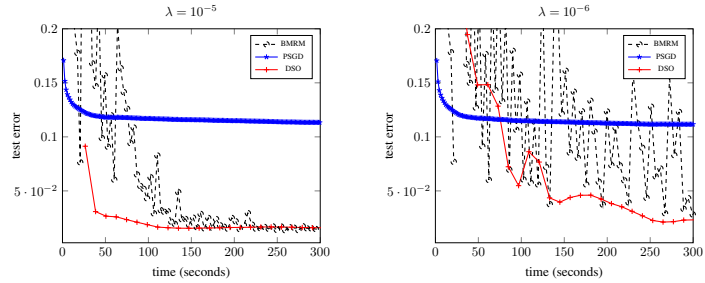

  \centering
  \testplot{svm}{webspamtrigram}{1e-5}{0.000010}{$10^{-5}$}{50}{300}{0.005}{0.5}{0.002}{0.2}
  \testplot{svm}{webspamtrigram}{1e-6}{0.000001}{$10^{-6}$}{50}{300}{0.005}{0.5}{0.002}{0.2}
  \caption{Test errors of different parallel optimization algorithms on
    linear SVM with \texttt{webspam-t} dataset, as a function of elapsed
    time.}
  \label{fig:par_svm_compare}
\end{figure}

\subsection{Terascale Learning with DSO}
\label{sec:TerascLearnwith}

Next, we demonstrate the scalability of DSO on one of the largest
publicly available datasets. Following the same experimental setup as
\citet{AgaChaDudetal14}, we work with the splice site recognition
dataset \cite{SonFra10} which contains 50 million training data points,
each of which has around 11.7 million dimensions. Each datapoint has
approximately 2000 non-zero coordinates and the entire dataset requires
around \textbf{3 TB} of storage. Previously \cite{SonFra10}, it has been shown
that sub-sampling reduces performance, and therefore we need to use the
entire dataset for training.

Similar to \citet{AgaChaDudetal14}, our goal is not to show the best
classification accuracy on this data (this is best left to domain
experts and feature designers). Instead, we wish to demonstrate the
scalability of DSO and establish that a) it can scale to such massive
datasets, and b) the empirical performance as measured by AUPRC (Area
Under Precision-Recall Curve) improves as a function of time.
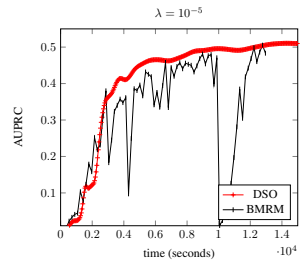
\begin{wrapfigure}{r}{0.5\textwidth}
  \tikzsetnextfilename{nips15_dnaexp}
  \begin{center}
    \begin{tikzpicture}[scale=0.46]
      \begin{axis}[xlabel=time (seconds),
        ylabel={AUPRC}, 
        enlargelimits=false, 
        legend pos=south east,
        restrict y to domain=0:0.55,
        title={$\lambda = 10^{-5}$}, 
        ymin={0.01}, ymax={0.55}, xmin={0}, xmax={15000}]
        
        \addplot[mark=+,color=red,thick]
        table[x index=2, y index=8, header=false, col sep=comma]
        {./dnaexp_result.txt};

        \addlegendentry{DSO}          

        \addplot[mark=|,color=black,thick]
        table[x index=1, y index=11, header=false, col sep=space]
        {./dnaexp_result2.txt};
        \addlegendentry{BMRM}         
      \end{axis}
    \end{tikzpicture}
  \end{center}
  \caption{AUPRC (Area Under Precision-Recall Curve) as a function of
    elapsed time on linear SVM with splice site recognition dataset.}
  \label{fig:dnaexp}
  \vspace{-0.2in}
\end{wrapfigure}


We used 14 machines with 8 cores per machine to train a linear SVM, and
plot AUPRC as a function of time in Figure~\ref{fig:dnaexp}. Since PSGD
did not perform well in earlier experiments, here we restrict our
comparison to BMRM. This experiment demonstrates one of the advantages
of stochastic optimization, namely that the test performance increases
steadily as a function of the number of iterations. On the other hand,
for a batch solver like BMRM the AUPRC fluctuates as a function of the
iteration number. The practical consequence of this observation is that,
one usually needs to wait for a batch optimizer to converge before using
the resulting solution. On the other hand, even the partial solutions
produced by a stochastic optimizer such as DSO usually exhibit good
generalization properties.

\section{Discussion and Conclusion}
\label{sec:DiscussionConclusion}

We presented a new reformulation of regularized risk minimization as a
saddle point problem, and showed that one can derive an efficient
distributed stochastic optimizer (DSO). We also proved rates of
convergence of DSO.  Unlike other solvers, our algorithm does not
require strong convexity and thus has wider applicability.  Our
experimental results show that DSO is competitive with state-of-the-art
optimizers such as BMRM and SGD, and outperforms simple parallel
stochastic optimization algorithms such as PSGD. 

A natural next step is to derive an asynchronous version of DSO
algorithm along the lines of the NOMAD algorithm proposed by
\citet{YunYuHsietal13}. We believe that our convergence proof which
only relies on having an equivalent serial sequence of updates will
still apply. Of course, there is also more room to further improve the
performance of DSO by deriving better step size adaptation schedules,
and exploiting memory caching to speed up random access. 
\newpage
\bibliographystyle{abbrvnat}
{\small
\bibliography{rerm}

\begin{thebibliography}{34}
\providecommand{\natexlab}[1]{#1}
\providecommand{\url}[1]{\texttt{#1}}
\expandafter\ifx\csname urlstyle\endcsname\relax
  \providecommand{\doi}[1]{doi: #1}\else
  \providecommand{\doi}{doi: \begingroup \urlstyle{rm}\Url}\fi

\bibitem[Agarwal et~al.(2014)Agarwal, Chapelle, Dud\'{i}k, and
  Langford]{AgaChaDudetal14}
A.~Agarwal, O.~Chapelle, M.~Dud\'{i}k, and J.~Langford.
\newblock A reliable effective terascale linear learning system.
\newblock \emph{JMLR}, 15:\penalty0 1111--1133, 2014.

\bibitem[Bertsekas and Tsitsiklis(1997)]{BerTsi97a}
D.~Bertsekas and J.~Tsitsiklis.
\newblock \emph{Parallel and Distributed Computation: Numerical Methods}.
\newblock 1997.

\bibitem[Bottou and Bousquet(2011)]{BotBou11}
L.~Bottou and O.~Bousquet.
\newblock The tradeoffs of large-scale learning.
\newblock \emph{Optimization for Machine Learning}, 2011.

\bibitem[Boyd and Vandenberghe(2004)]{BoyVan04}
S.~Boyd and L.~Vandenberghe.
\newblock \emph{Convex Optimization}.
\newblock 2004.

\bibitem[Boyd et~al.(2010)Boyd, Parikh, Chu, Peleato, and
  Eckstein]{BoyParChuPelEtal10}
S.~Boyd, N.~Parikh, E.~Chu, B.~Peleato, and J.~Eckstein.
\newblock Distributed optimization and statistical learning via the alternating
  direction method of multipliers.
\newblock \emph{Found.\ and Trends in ML}, 3\penalty0 (1):\penalty0 1--123,
  2010.

\bibitem[Bradley et~al.(2011)Bradley, Kyrola, Bickson, and
  Guestrin]{BraKyrBicGue11}
J.~Bradley, A.~Kyrola, D.~Bickson, and C.~Guestrin.
\newblock Parallel coordinate descent for {L1}-regularized loss minimization.
\newblock In \emph{ICML}, pages 321--328, 2011.

\bibitem[Chu et~al.(2006)Chu, Kim, Lin, Yu, Bradski, Ng, and
  Olukotun]{ChuKimLinYuetal06}
C.-T. Chu, S.~K. Kim, Y.-A. Lin, Y.~Yu, G.~Bradski, A.~Y. Ng, and K.~Olukotun.
\newblock Map-reduce for machine learning on multicore.
\newblock In \emph{NIPS}, pages 281--288, 2006.

\bibitem[Duchi et~al.(2010)Duchi, Hazan, and Singer]{DucHazSin10}
J.~Duchi, E.~Hazan, and Y.~Singer.
\newblock Adaptive subgradient methods for online learning and stochastic
  optimization.
\newblock \emph{JMLR}, 12:\penalty0 2121--2159, 2010.

\bibitem[Fan et~al.(2008)Fan, Chang, Hsieh, Wang, and Lin]{FanChaHsiWanetal08}
R.-E. Fan, J.-W. Chang, C.-J. Hsieh, X.-R. Wang, and C.-J. Lin.
\newblock {LIBLINEAR}: A library for large linear classification.
\newblock \emph{JMLR}, 9:\penalty0 1871--1874, Aug. 2008.

\bibitem[Gemulla et~al.(2011)Gemulla, Nijkamp, Haas, and
  Sismanis]{GemNijHaaSis11}
R.~Gemulla, E.~Nijkamp, P.~J. Haas, and Y.~Sismanis.
\newblock Large-scale matrix factorization with distributed stochastic gradient
  descent.
\newblock In \emph{KDD}, pages 69--77, 2011.

\bibitem[Hastie et~al.(2009)Hastie, Tibshirani, and Friedman]{HasTibFri09}
T.~Hastie, R.~Tibshirani, and J.~Friedman.
\newblock \emph{The Elements of Statistical Learning}.
\newblock 2009.

\bibitem[Ho et~al.(2013)Ho, Cipar, Cui, Lee, Kim, Gibbons, Gibson, Ganger, and
  Xing]{HoCiCuLeeetal13}
Q.~Ho, J.~Cipar, H.~Cui, S.~Lee, J.~K. Kim, P.~B. Gibbons, G.~A. Gibson,
  G.~Ganger, and E.~P. Xing.
\newblock More effective distributed {ML} via a stale synchronous parallel
  parameter server.
\newblock In \emph{NIPS}, 2013.

\bibitem[Hsieh et~al.(2008)Hsieh, Chang, Lin, Keerthi, and
  Sundararajan]{HsiChaLinKeeetal08}
C.~J. Hsieh, K.~W. Chang, C.~J. Lin, S.~S. Keerthi, and S.~Sundararajan.
\newblock A dual coordinate descent method for large-scale linear {SVM}.
\newblock In \emph{ICML}, pages 408--415, 2008.

\bibitem[Hsieh et~al.(2015)Hsieh, Yu, and Dhillon]{HsiYuDhi15}
C.-J. Hsieh, H.-F. Yu, and I.~S. Dhillon.
\newblock {PASSC}o{D}e: {P}arallel {AS}ynchronous {S}tochastic dual
  {C}oordinate {D}escent.
\newblock In \emph{ICML}, 2015.

\bibitem[Johnson and Zhang(2013)]{JohZha13}
R.~Johnson and T.~Zhang.
\newblock Accelerating stochastic gradient descent using predictive variance
  reduction.
\newblock In \emph{NIPS}, pages 315--323, 2013.

\bibitem[Langford et~al.(2009)Langford, Smola, and Zinkevich]{LanSmoZin09}
J.~Langford, A.~J. Smola, and M.~Zinkevich.
\newblock Slow learners are fast.
\newblock In \emph{NIPS}, 2009.

\bibitem[Liu and Nocedal(1989)]{LiuNoc89}
D.~C. Liu and J.~Nocedal.
\newblock On the limited memory {BFGS} method for large scale optimization.
\newblock \emph{Mathematical Programming}, 45\penalty0 (3):\penalty0 503--528,
  1989.

\bibitem[Nedi\'{c} and Bertsekas(2001)]{NedBer01}
A.~Nedi\'{c} and D.~P. Bertsekas.
\newblock Incremental subgradient methods for nondifferentiable optimization.
\newblock \emph{SIAM Journal on Optimization}, 12\penalty0 (1):\penalty0
  109--138, 2001.

\bibitem[Nemirovski et~al.(2009)Nemirovski, Juditsky, Lan, and
  Shapiro]{NemJudLanSha09}
A.~Nemirovski, A.~Juditsky, G.~Lan, and A.~Shapiro.
\newblock Robust stochastic approximation approach to stochastic programming.
\newblock \emph{SIAM J. on Optimization}, 19\penalty0 (4):\penalty0 1574--1609,
  Jan. 2009.

\bibitem[Nesterov(2004)]{Nes04}
Y.~Nesterov.
\newblock \emph{Introductory Lectures On Convex Optimization: A Basic Course}.
\newblock Springer, 2004.

\bibitem[Recht and R\'e(2013)]{RecRe13}
B.~Recht and C.~R\'e.
\newblock Parallel stochastic gradient algorithms for large-scale matrix
  completion.
\newblock \emph{Mathematical Programming Computation}, 5\penalty0 (2):\penalty0
  201--226, June 2013.

\bibitem[Recht et~al.(2011)Recht, Re, Wright, and Niu]{RecReWriNiu11}
B.~Recht, C.~Re, S.~Wright, and F.~Niu.
\newblock Hogwild: {A} lock-free approach to parallelizing stochastic gradient
  descent.
\newblock In \emph{NIPS}, pages 693--701, 2011.

\bibitem[Sch{\"o}lkopf and Smola(2002)]{SchSmo02}
B.~Sch{\"o}lkopf and A.~J. Smola.
\newblock \emph{Learning with Kernels}.
\newblock 2002.

\bibitem[Shalev-Shwartz and Ben-David(2014)]{ShaBen14}
S.~Shalev-Shwartz and S.~Ben-David.
\newblock \emph{Understanding Machine Learning}.
\newblock 2014.

\bibitem[Shalev-Shwartz et~al.(2007)Shalev-Shwartz, Singer, and
  Srebro]{ShaSinSre07}
S.~Shalev-Shwartz, Y.~Singer, and N.~Srebro.
\newblock Pegasos: Primal estimated sub-gradient solver for {SVM}.
\newblock In \emph{ICML}, 2007.

\bibitem[Smola and Narayanamurthy(2010)]{SmoNar10}
A.~J. Smola and S.~Narayanamurthy.
\newblock An architecture for parallel topic models.
\newblock In \emph{VLDB}, 2010.

\bibitem[Sonnenburg and Franc(2010)]{SonFra10}
S.~Sonnenburg and V.~Franc.
\newblock {COFFIN:} a computational framework for linear {SVMs}.
\newblock In \emph{ICML}, 2010.

\bibitem[Teo et~al.(2010)Teo, Vishwanthan, Smola, and Le]{TeoVisSmoLe10}
C.~H. Teo, S.~V.~N. Vishwanthan, A.~J. Smola, and Q.~V. Le.
\newblock Bundle methods for regularized risk minimization.
\newblock \emph{JMLR}, 11:\penalty0 311--365, January 2010.

\bibitem[Webb et~al.(2006)Webb, Caverlee, and Pu]{WebCavPu06}
S.~Webb, J.~Caverlee, and C.~Pu.
\newblock Introducing the webb spam corpus: Using email spam to identify web
  spam automatically.
\newblock In \emph{CEAS}, 2006.

\bibitem[Yan et~al.(2009)Yan, Xu, and Qi]{YanXuQi09}
F.~Yan, N.~Xu, and Y.~Qi.
\newblock Parallel inference for latent {D}irichlet allocation on graphics
  processing units.
\newblock In \emph{NIPS}, pages 2134--2142. 2009.

\bibitem[Yang(2013)]{Yan13}
T.~Yang.
\newblock Trading computation for communication: Distributed stochastic dual
  coordinate ascent.
\newblock In \emph{NIPS}, 2013.

\bibitem[Yun et~al.(2014)Yun, Yu, Hsieh, Vishwanathan, and
  Dhillon]{YunYuHsietal13}
H.~Yun, H.-F. Yu, C.-J. Hsieh, S.~V.~N. Vishwanathan, and I.~S. Dhillon.
\newblock {NOMAD}: Non-locking, st{O}chastic {M}ulti-machine algorithm for
  {A}synchronous and {D}ecentralized matrix completion.
\newblock \emph{VLDB}, 2014.

\bibitem[Zhang and Xiao(2015)]{ZhaXia15}
Y.~Zhang and L.~Xiao.
\newblock Di{SCO}: {D}istributed optimization for {S}elf-{C}oncordant empirical
  loss.
\newblock In \emph{ICML}, 2015.

\bibitem[Zinkevich et~al.(2010)Zinkevich, Smola, Weimer, and Li]{ZinSmoWeiLi10}
M.~Zinkevich, A.~J. Smola, M.~Weimer, and L.~Li.
\newblock Parallelized stochastic gradient descent.
\newblock In \emph{NIPS}, pages 2595--2603, 2010.

\end{thebibliography}
}
\clearpage
\appendix









    

    
    



\section{Dataset and Implementation Details}
\label{sec:DatasetDetails}

\begin{table}
  \centering
  \caption{Summary of the datasets used in our experiments. $m$ is the
    total \# of examples, $d$ is the \# of features, $s$ is the
    feature density (\% of features that are non-zero).
    K/M/G denotes a thousand/million/billion.  }
  {\small
  \begin{tabular}{|l|r|r|r|r|} 
    \hline
    \textbf{Name} & $m$\phantom{.aa} &
    $d$\phantom{aaa} & $\abr{\Omega}$ & $s$(\%) \\
    \hline
    {\small \texttt{real-sim}}& 57.76K & 20.95K   & 2.97M & 0.245 
    \\ \hline
    {\small \texttt{webspam-t}}  & 350.00K  & 16.61M & 1.28G & 0.022
    \\ \hline
   \end{tabular}
   }
  \label{tab:data}
\end{table}

We implemented DSO, SGD, and PSGD ourselves, while for BMRM we used
the optimized implementation that is available from the toolkit for
advanced optimization
(TAO)(\url{https://bitbucket.org/sarich/tao-2.2}).  All algorithms are
implemented in \texttt{C++} and use \texttt{MPI} for communication.
In our multi-machine experiments, each algorithm was run on four
machines with eight cores per machine.  DSO, SGD, and PSGD used
AdaGrad \citep{DucHazSin10} step size adaptation. We also used
Stochastic Variance Reduced Gradient (SVRG) of \citet{JohZha13} to
accelerate updates of DSO.  In the multi-machine setting DSO
initializes parameters of each MPI process by locally executing twenty
iterations of dual coordinate descent \citep{FanChaHsiWanetal08} on
its local data to locally initialize $w_j$ and $\alpha_i$ parameters;
then $w_j$ values were averaged across machines.
We chose binary logistic regression and SVM as test problems, i.e.,
$\phi_j(s)=\half s^2$ and $\ell_i(u)=\log(1+\exp(-u)),[1-u]_+$. To prevent
degeneracy in logistic regression, values of $\alpha_{i}$'s are
restricted to $(10^{-14}, 1-10^{-14})$, while in the case of linear
SVM they are restricted to $[0, 1]$.  Similarly, the $w_{j}$'s are
restricted to lie in the interval $[-1/\sqrt{\lambda},
  1/\sqrt{\lambda}]$ for linear SVM and $[-\sqrt{\log(2)/\lambda},
  \sqrt{\log(2)/\lambda}]$ for logistic regression, following the idea
of \citet{ShaSinSre07}.

\section{Proofs}
\label{sec:Proofs}

Let $\rbr{\wb^{t}, \alphab^{t}}$ denote the parameter vector after the
$t$-th epoch. Without loss of generality, we will focus on the inner
iterations of the $\rbr{t+1}$-st epoch. Consider a time instance at
which $k$ updates corresponding to $\rbr{i_{1}, j_{1}}, \rbr{i_{2},
  j_{2}}, \ldots, \rbr{i_{k}, j_{k}}$ have been performed on
$\rbr{\wb^{t}, \alphab^{t}}$, which results in the parameter values
denoted by $\rbr{\wb_{k}^{t}, \alphab_{k}^{t}}$. In terms of analysis,
it is useful to recognize that there is a natural ordering of these
updates as follows: $\rbr{i_{a}, j_{a}}$ appears before $\rbr{i_{b},
  j_{b}}$ if updates to $\rbr{w_{j_{a}}, \alpha_{i_{a}}}$ were
performed before updating $\rbr{w_{j_{b}}, \alpha_{i_{b}}}$. On the
other hand, if $\rbr{w_{j_{a}}, \alpha_{i_{a}}}$ and $\rbr{w_{j_{b}},
  \alpha_{i_{b}}}$ were updated at the same time because we have $p$
processors simultaneously updating the parameters, then the updates
are ordered according to the rank of the processor performing the
update\footnote{Any other tie-breaking rule would also suffice.}. The
following lemma asserts that the updates are serializable in the sense
that $\rbr{\wb^{t}_{k}, \alphab^{t}_{k}}$ can be recovered by
performing $k$ serial updates on function $f_{k}$ which is defined
below.

\begin{lemma}  
  For all $k$ and $t$ we have
  \begin{align}
    \label{def3}
    \wb^t_k = \wb^t_{k-1} -\eta_t \nabla_{\wb} f_k\rbr{\wb^t_{k-1}
      ,\alphab^t_{k-1}} \text{ and }
    \alphab^t_k = \alphab^t_{k-1} -\eta_t \nabla_{\alphab}
    {-f}_k\rbr{\wb^t_{k-1} ,\alphab^t_{k-1}}, 
  \end{align}
  where 
  \begin{align*}
    f_{k} (\wb, \alphab) := f_{i_{k}, j_{k}}\rbr{w_{j_k}, \alpha_{i_k}}.
  \end{align*}  
\end{lemma}
\begin{proof}
  Let $q$ be the processor which performed the $k$-th update in the
  $\rbr{t+1}$-st epoch. Moreover, let $\rbr{k-\delta}$ be the most
  recent pervious update done by processor $q$.  There exists $\delta
  \ge \delta',\delta'' \ge 1$ such that $\rbr{\wb^t_{k-\delta'},
    \alphab^t_{k-\delta''}}$ be the parameter values read by the $q$-th
  processor to the perform $k$-th update.  Because of our data
  partitioning scheme, only $q$ can change the value of the $i_k$-th
  component of $\alphab$ and the $j_k$-th component of $\wb$. Therefore,
  we have
  \begin{align}
    \alpha^t_{k-1,i_k} = \alpha^t_{\kappa,i_k},  \qquad k-\delta \le
    \kappa \le k-1,   \\ 
    w_{k-1, j_k} = w^t_{\kappa, j_{k}}, \qquad k-\delta \le \kappa \le
    k-1.
  \end{align}
  Since $f_k$ is invariant to changes in any coordinate other than
  $\rbr{i_k,j_k}$, we have
  \begin{align}
    f_k\rbr{\wb^t_{k-\delta'} ,\alphab^t_{k-\delta''}} =
    f_k\rbr{\wb^t_{k-1} ,\alphab^t_{k-1}}.  
  \end{align}
  The claim holds because we can write the $k$-th update formula as
  \begin{align}
    \wb^t_k & = \wb^t_{k-1} -\eta_t \nabla_{\wb}
    f_k\rbr{\wb^t_{k-\delta'} ,\alphab^t_{k-\delta''}} \text{ and } \\  
    \alphab^t_k & = \alphab^t_{k-1} -\eta_t \nabla_{\alphab}
    {-f}_k\rbr{\wb^t_{k-\delta'} ,\alphab^t_{k-\delta''}}. 
  \end{align}
\end{proof}
As a consequence of the above lemma, it suffices to analyze the serial
convergence of the function $f(\wb, \alphab) := \sum_{k} f_{k}(\wb,
\alphab)$. Towards this end, we first prove the following technical
lemma. Note that it is closely related to general results on convex
functions ( e.g., Theorem 3.2.2 in \cite{Nes04}, Lemma 14. 1. in \cite{ShaBen14} ).
\begin{lemma}
  Suppose there exists $C>0$ and $D>0$ such that for all
  $\rbr{\wb, \alphab}$ and $\rbr{\wb', \alphab'}$ we have $\nbr{\wb -
    \wb'}^{2} + \nbr{\alphab - \alphab'}^{2} \leq D$, and for all $t =
  1, \ldots, T$ and all $\rbr{\wb,\alphab}$ we have
  \begin{align}
    \label{eq:key-lemma}
    \nbr{\wb^{t+1} - \wb}^{2} + \nbr{\alphab^{t+1} - \alphab}^{2} \le
    \nbr{\wb^{t} - \wb}^{2} + \nbr{\alphab^{t} - \alphab}^{2} - 2
    \eta_{t} \rbr{f\rbr{\wb^{t},\alphab} - f\rbr{\wb,\alphab^{t}}} + C\eta_{t}^{2}
   ,
  \end{align}
  then setting $\eta_{t} = \sqrt{\frac{D}{2Ct}}$ ensures that 
  \begin{align}
    \label{eq:rate}
    \varepsilon\rbr{\tilde{\wb}^{T}, \tilde{\alphab}^{T}} & \le
    \sqrt{\frac{2DC}{T}}.
 \end{align}
\end{lemma}
\begin{proof}
  Rearrange \eqref{eq:key-lemma} and divide by $\eta_{t}$ to obtain
  \begin{align*}
    2\rbr{f\rbr{\wb^{t},\alphab} - f\rbr{\wb,\alphab^{t}}} \le
    \eta_{t} C + \frac{1}{\eta_t} 
    \Big(
    \nbr{\wb^{t} - \wb} ^ 2  + \nbr{\alphab^{t} -
        \alphab}^{2} - \nbr{\wb^{t+1} - \wb}^{2} - \nbr{\alphab^{t+1} -
        \alphab}^{2} \Big).
  \end{align*}
  Summing the above for $t=1,\ldots, T$ yields 
  \begin{align}
    \nonumber
    2\sum_{t=1}^{T} & f\rbr{\wb^{t},\alphab} - 2\sum_{t=1}^{T} f\rbr{\wb,\alphab^{t}}
    \le  
    \sum_{t=1}^{T} \eta_{t} C + \frac{1}{\eta_{1}}\rbr{\nbr{\wb^{1} - \wb}^{2} +
      \nbr{\alphab^{1} - \alphab}^{2}} \\
    \nonumber
    &+ \sum_{t=2}^{T-1} \rbr{ \frac{1}{\eta_{t+1}} -
      \frac{1}{\eta_{t}}} \rbr{ \nbr{\wb^{t} - \wb}^{2} +
      \nbr{\alphab^{t} -
        \alphab}^{2}}\\
    \nonumber 
    &- \frac{1}{\eta_{T}} \rbr{ \nbr{\wb^{T+1} - \wb}^{2} + \nbr{\alphab^{T+1}
        - \alphab}^{2}} \\ 
    \nonumber
    &\le \sum_{t=1}^{T} \eta_{t} C + 
    \frac{1}{\eta_{1}} D + \sum_{t=2}^{T-1} \rbr{ \frac{1}{\eta_{t+1}} -
      \frac{1}{\eta_{t}}} D \\ 
    \label{eq:int-bound}
    &\le \sum_{t=1}^{T} \eta_{t} C + \frac{1}{\eta_{T}} D.
  \end{align}
  Thanks to convexity in $\wb$ and concavity in $\alphab$
  \begin{align}
    \label{eq:avg-w}
    f\rbr{\tilde{\wb}^{T},\alphab} = f\rbr{\frac{1}{T}\sum_{t=1}^{T}
      \wb^{t}, \alphab}
    \le \frac{1}{T} \sum_{t=1}^{T} f\rbr{\wb^{t},\alphab}, \text{ and }  \\
    \label{eq:avg-alpha}
    -f\rbr{\wb,\tilde{\alphab}^{T}} = -f\rbr{\wb, \frac{1}{T}
      \sum_{t=1}^{T} \alphab^{t}} \le \frac{1}{T} \sum_{t=1}^{T}
    {-f}\rbr{\wb,\alphab^{t}}.
  \end{align}
  Substituting \eqref{eq:avg-w} and \eqref{eq:avg-alpha} into
  \eqref{eq:int-bound}, and letting $\eta_{t} = \sqrt{\frac{D}{2Ct}}$
  leads to the following sequence of inequalities
  \begin{align*}
    f\rbr{\tilde{\wb}^{T},\alphab} - f\rbr{\wb,\tilde{\alphab}^{T}} \le
    \frac{\sum_{t=1}^{T} \eta_{t} C + \frac{1}{\eta_{T}} D}{2T} \leq 
    \frac{\sqrt{DC}}{2T} \sum_{t=1}^{T} \frac{1}{\sqrt{2t}} + \sqrt{\frac{DC}{2T}}. 
  \end{align*}
  The claim in \eqref{eq:rate} follows by using $\sum_{t=1}^{T}
  \frac{1}{\sqrt{2t}} \leq \sqrt{2T}$.
\end{proof}

To prove convergence of DSO it suffices to show that it satisfies
\eqref{eq:key-lemma}. In order to derive \eqref{eq:thm-rate}, $C$ of
\eqref{eq:key-lemma} has to be the order of $d$. In case of
$L_2$-regularizer, it has to be dependent only on $c$ to obtain \eqref{eq:thm-special-rate}. The proof is
related to techniques outlined in \citet{NedBer01}.
\begin{lemma}
  Under the assumptions outlined in Theorem~\ref{thm:convergence},
  Algorithm~\ref{alg:synchronous} satisfies \eqref{eq:key-lemma} with
  $C$ of the form of $C=C_1 d$. It does with $C=C_2$ in case of $L_2$-regularizer.
  Here $C_1$ and $C_2$ is dependent only on $c$.
\end{lemma}
\begin{proof}
For $\wb$,
  \begin{align*}
    \nbr{\wb_{k}^{t} - \wb}^{2}
    & = \nbr{\wb_{k-1}^{t} -\eta_{t} \nabla_{\wb} f_{k} \rbr{\wb_{k-1}^{t}, \alphab_{k-1}^{t}} - \wb}^{2}\\
    & = \nbr{\wb_{k-1}^{t} - \wb}^{2} - 2 \eta_{t} \inner{\nabla_{\wb}
      f_{k}\rbr{\wb_{k-1}^{t}, \alphab_{k-1}^{t}}}{\wb_{k-1}^{t} - \wb}
    + \eta_{t}^{2} \nbr{\nabla_{\wb} f_{k}\rbr{\wb_{k-1}^{t}, \alphab_{k-1}^{t}}}^{2} \\
    & \le \nbr{\wb_{k-1}^{t} - \wb}^{2} - 2 \eta_{t}
    \rbr{f_{k}\rbr{\wb_{k-1}^{t}, \alphab_{k-1}^{t}} - f_{k} \rbr{\wb, \alphab_{k-1}^{t}}}
    + \eta_{t}^{2} \nbr{\nabla_{\wb} f_{k}\rbr{\wb_{k-1}^{t}, \alphab_{k-1}^{t}}}^{2}.
  \end{align*}
  Analogously for $\alphab$ we have
  \begin{align*}
    \nbr{\alphab_{k}^{t} - \alphab}^{2} &\le \nbr{ \alphab_{k-1}^t -
      \alphab}^2 - 2 \eta_{t} \rbr{-f_{k}\rbr{\wb_{k-1}^{t},
        \alphab_{k-1}^{t}} + f_{k}\rbr{\wb_{k-1}^{t}, \alphab}} 
    + \eta_{t}^{2} \nbr{\nabla_{\val} {-f}_{k}\rbr{\wb_{k-1}^{t}, \alphab_{k-1}^{t}}}^{2}. 
  \end{align*}
  Adding the above two inequalities, rearranging
  \begin{align*}
    \nbr{\wb_{k}^{t} - \wb}^{2} + \nbr{\alphab_{k}^{t} - \alphab}^{2} -
    \nbr{\wb_{k-1}^{t} - \wb}^{2} - \nbr{\alphab_{k-1}^{t} -
      \alphab}^{2} \le & - 2 \eta_{t} \rbr{f_{k}\rbr{\wb_{k-1}^{t},
        \alphab} -f_k\rbr{\wb, \alphab_{k-1}^{t}}} \\
    & + \eta_{t}^{2} \nbr{\nabla_{\wb} f_{k}\rbr{\wb_{k-1}^{t}, \alphab_{k-1}^{t}}}^{2} \\ 
    & + \eta_{t}^{2} \nbr{\nabla_{\val} {-f}_{k}\rbr{\wb_{k-1}^{t}, \alphab_{k-1}^{t}}}^{2},  
  \end{align*}
  and summing the above equation for $k=1, \ldots, \abr{\Omega}$ obtains 
  \begin{align}
    \nonumber
    & \nbr{\wb^{t+1} - \wb}^{2} + \nbr{\alphab^{t+1} - \alphab}^{2} -
      \nbr{\wb^t - \wb}^{2} - \nbr{\alphab^{t} - \alphab}^{2} \\
    \le & - 2 \eta_t \rbr{ f(\wb^t, \alphab) - f\rbr{\wb,\alphab^t} }  \nonumber \\
    & \label{eq:key_alpha}
    +2\eta_t\sum_{k}\rbr{ \rbr{-f_{k}\rbr{\wb,\alphab^t}} - \rbr{-f_{k}\rbr{\wb,\alphab^t_{k-1}}}} \\  
    \label{eq:key_alpha_sq}
    &  + \eta_{t}^{2}\sum_{k} \nbr{\nabla_{\val} f_{k}\rbr{\wb_{k-1}^{t}, \alphab_{k-1}^{t}}}^{2}\\
    \label{eq:key_w}
    & +2 \eta_{t} \rbr{\sum_{k} f_{k}\rbr{\wb^t,\alphab}-f_{k}\rbr{\wb^t_{k-1}, \alphab}                               
                              +\half\eta_t \nbr{\nabla_{\wb} f_{k}\rbr{\wb_{k-1}^{t}, \alphab_{k-1}^{t}}}^{2}
                              }.
  \end{align}
  In the following, we derive upper bounds for each term of \eqref{eq:key_alpha}, \eqref{eq:key_alpha_sq} and \eqref{eq:key_w}.
  First observe that 

\begin{align}
    \rbr{{-f}_{k}\rbr{\wb,\alphab^t}} - \rbr{{-f}_{k}\rbr{\wb, \alphab^t_{k-1}} } 
    &\le \inner{\nabla_{\val} {-f}_k(\wb,\alphab^t) }{ \alphab^t_{k-1} -\alphab^t } \\
    &= \inner{\nabla_{\val} {-f}_k(\wb,\alphab^t) }{ \sum_{l:l<k-1} \eta_t\nabla_{\val} {-f}_l(\wb^t_{l-1},\alphab^t_{l-1}) } \\
    &\le \nbr{\nabla_{\val} {-f}_k(\wb,\alphab^t) } \nbr{ \sum_{l:l<k-1,i_l=i_k} \eta_t\nabla_{\val} {-f}_l(\wb^t_{l-1},\alphab^t_{l-1})}.
\end{align}
As we can see 
\begin{align}
  \label{eq:upbd_nabla_alpha}
  \nbr{\nabla_{\val} {-f}_k(\wb,\alphab^t) }
  &\le \frac{c}{m\abr{\Omega_{i_k}}} + \frac{\abr{w^t_{k,j_k} x_{i_k,j_k}}}{m} \\
  \label{eq:upbd_nabla_alpha2}
  &\le \frac{c}{m}\rbr{\frac{1}{\abr{\Omega_{i_k}}} + \abr{x_{i_k,j_k}}}, 
\end{align}
using \eqref{eq:upbd_nabla_alpha}, we have
\begin{align}
  \sum_k \nbr{\nabla_{\val} {-f}_k(\wb,\alphab^t) }
  \le  c + \frac{1}{m}\sum_{i=1}^m \sum_{j=1}^d\abr{w_j x_{ij}} 
  \le  c + \frac{1}{m}\sum_{i=1}^m \nbr{\wvec} \nbr{\xvec_i}
  \le  2 c^2 .
\end{align}
On the other hand, using \eqref{eq:upbd_nabla_alpha2}, we have
\begin{align}
\nbr{ \sum_{l:l<k-1,i_l=i_k} \eta_t\nabla_{\val} {-f}_l(\wb^t_{l-1},\alphab^t_{l-1}) }
  &\le \frac{c\eta_t}{m} \rbr{1 + \sum_{j\in \Omega_i} \abr{x_{ij}}} \\ 
  &\le \frac{c\eta_t}{m} \rbr{1 + \sqrt{\abr{\Omega_i}} \nbr{\xvec_i}} \quad \text{(Cauchy-Schwartz)}\\
  &\le \frac{c^2\eta_t}{m} \rbr{1 + \sqrt{\maxOmegai} }. 
\end{align}
Therefore, we can conclude
\begin{align}
  \sum_k \rbr{{-f}_{k}\rbr{\wb,\alphab^t}} - \rbr{{-f}_{k}\rbr{\wb, \alphab^t_{k-1}} }
  &= \sum_k \nbr{\nabla_{\val} {-f}_k(\wb,\alphab^t) } \nbr{ \sum_{l:l<k-1,i_l=i_k} \eta_t\nabla_{\val} {-f}_l(\wb^t_{l-1},\alphab^t_{l-1}) } \\
  & \le \sum_k \nbr{\nabla_{\val} {-f}_k(\wb,\alphab^t) } \cdot \frac{c^2\eta_t}{m} \rbr{1 + \sqrt{\maxOmegai} } \\
  &\le \frac{2c^4\eta_t^2\sqrt{\maxOmegai}}{m}. 
\end{align}
Also for the term \eqref{eq:key_alpha_sq},
\begin{align}
  \sum_k \nbr{\nabla_{\val} {-f}_k(\wb^t_{k-1},\alphab^t_{k-1}) }^2
  &\le \sum_k \frac{c^2}{m^2}\rbr{\frac{1}{\abr{\Omega_{i_k}}} + \abr{x_{i_k,j_k}}}^2 \\
  &\le \frac{2c^2}{m^2} \sum_k \rbr{\frac{1}{\abr{\Omega_{i_k}}^2} + \abr{x_{i_k,j_k}}^2  } \\
  &\le \frac{2c^2}{m^2} \sum_i \rbr{\frac{1}{\abr{\Omega_i}} + \nbr{\xb_{i}}^2  } \\
  &\le \frac{4c^4}{m}.
\end{align}

Similarly for \eqref{eq:key_w},
  \begin{align*}
    f_{k}\rbr{\wb^{t}, \alphab} - f_{k}\rbr{\wb_{k-1}^{t}, \alphab} &
    \le \inner{\nabla_{\wb} f_{k}\rbr{\wb^{t}, \alphab}}{\wb_{k-1}^{t} - \wb^{t}} \\
  &= \inner{\nabla_{\wb} f_k (\wb^t, \alphab ) }{\eta_t \sum_{l:l<k-1}\nabla_{\wb} f_l (\wb^t_l, \alphab^t_l)} \\
  &\le \eta_t \sum_{l:l<k-1, j_l = j_k } \abr{ \frac{\lambda \nabla \phi(w^t_{j_k})}{\abs{\Omegabar_j}} - \frac{\alpha_{i_k} x_{i_k,j_k}}{m} }  
         \abr{\frac{\lambda \nabla \phi(w^t_{l, j_l})}{\abs{\Omegabar_j}} - \frac{\alpha_{i_l} x_{i_l,j_l}}{m} }  \\
  &\le c^2 \eta_t \sum_{l:l<k-1, j_l = j_k } \abr{ \frac{\lambda}{\abs{\Omegabar_j}} + \frac{\abs{ x_{i_k,j_k}}}{m} }  
         \abr{\frac{\lambda}{\abs{\Omegabar_j}} + \frac{\abs{x_{i_l,j_l}}}{m} },
  \end{align*}
  which leads to 
  \begin{align}
    & f_{k} \rbr{\wb^t, \alphab} - f_{k}\rbr{\wb^t_{k-1},\alphab} + \half \eta_t  \nbr{\nabla_{\wb} f_{k}\rbr{\wb_{k-1}^{t}, \alphab_{k-1}^{t}}}^{2} \\
    &\le c^2 \eta_t \sum_{l:l<k-1, j_l = j_k } \abr{ \frac{\lambda}{\abs{\Omegabar_j}} + \frac{\abs{ x_{i_k,j_k}}}{m} }  
         \abr{\frac{\lambda}{\abs{\Omegabar_j}} + \frac{\abs{x_{i_l,j_l}}}{m} }
         +  c^2 \eta_t  \abr{ \frac{\lambda}{\abs{\Omegabar_j}} + \frac{\abs{ x_{i_k,j_k}}}{m} }  
         \abr{\frac{\lambda}{\abs{\Omegabar_j}} + \frac{\abs{x_{i_k,j_k}}}{m} } \\
    &\le c^2 \eta_t \sum_{i \in \Omegabar_{j_k} } \abr{ \frac{\lambda}{\abs{\Omegabar_j}} + \frac{\abs{ x_{i_k,j_k}}}{m} }  
         \abr{\frac{\lambda}{\abs{\Omegabar_j}} + \frac{\abs{x_{i,j_k}}}{m} }.
  \end{align}
  
Thus we can get the bound for \eqref{eq:key_w} as follows, 
\begin{align}
  & \sum_k f_{k}\rbr{\wb^t, \alphab} - f_{k}\rbr{\wb^t_{k-1},\alphab} + \half \eta_t  \nbr{\nabla_{\wb} f_{k}\rbr{\wb_{k-1}^{t}, \alphab_{k-1}^{t}}}^{2} \\
  &\le\sum_k c^2 \eta_t \sum_{i \in \Omegabar_{j_k} }
      \abr{ \frac{\lambda}{\abs{\Omegabar_j}} + \frac{\abs{ x_{i_k,j_k}}}{m} }  
      \abr{\frac{\lambda}{\abs{\Omegabar_j}} + \frac{\abs{x_{i,j_k}}}{m} }\\
  &= c^2 \eta_t \sum_j \sum_{i' \in \Omegabar_{j}} \sum_{i \in \Omegabar_{j} }
      \abr{ \frac{\lambda}{\abs{\Omegabar_j}} + \frac{\abs{ x_{i',j}}}{m} }  
      \abr{\frac{\lambda}{\abs{\Omegabar_j}} + \frac{\abs{x_{i,j}}}{m} }\\
  &= c^2 \eta_t \sum_j
      \rbr{\sum_{i' \in \Omegabar_{j}}  \abr{ \frac{\lambda}{\abs{\Omegabar_j}} + \frac{\abs{ x_{i',j}}}{m} } }
      \rbr{\sum_{i \in \Omegabar_{j}}  \abr{ \frac{\lambda}{\abs{\Omegabar_j}} + \frac{\abs{ x_{i,j}}}{m} } } \\
  &\le c^2 \eta_t \sum_j
      \abr{\Omegabar_{j}} \sum_{i \in \Omegabar_{j}}  \abr{ \frac{\lambda}{\abs{\Omegabar_j}} + \frac{\abs{ x_{i,j}}}{m} }^2 \\
  &\le 2c^2 \eta_t \sum_j
       \sum_{i \in \Omegabar_{j}} \rbr{  \frac{\lambda^2}{\abs{\Omegabar_j}} + \frac{\abs{\Omegabar_j}\abs{x_{i,j}}^2}{m^2} }\\
  &\le  2c^2 \eta_t \sum_{j} \lambda^2 + 2c^2 \eta_t  \sum_i \sum_{j \in \Omega_i} \frac{\abs{x_{i,j}}^2}{m} \\
  &\le  2c^4 d \eta_t + 2c^4 \eta_t .
\end{align}

Incorporating this bound into \eqref{eq:key_alpha}, \eqref{eq:key_alpha_sq} and \eqref{eq:key_w}, we get
\begin{align}
    &\nbr{\wb^{t+1} - \wb}^{2} + \nbr{\alphab^{t+1} - \alphab}^{2} -
     \nbr{\wb^{t} - \wb}^{2}  -  \nbr{\alphab^{t} - \alphab}^{2}    
    \le -2 \eta_t \rbr{ f\rbr{\wb^{t}, \alphab} - f\rbr{\wb, \alphab^{t}} } \nonumber \\
    \label{eq:factor_eta_t_sq}
    & + 2\eta_t \cdot \rbr{2c^4 d\eta_t + 2c^4 \eta_t }
      + \frac{2c^4\eta_t^2\sqrt{\maxOmegai}}{m}
      + \frac{4c^4}{m} \cdot \eta_t^2.
\end{align}
Thus we can get such a sufficiently large constant $C$
that $C d \eta_t^2$ upper bounds the term of \eqref{eq:factor_eta_t_sq}.

In the case of $\phi_j(s) = \half s^2$, we assume the following additionally $\eta_t < \frac{1}{\lambda}$ and $\sqrt{ \maxOmegai} < m$.
As we can see from the expression \eqref{eq:factor_eta_t_sq}, assuming $\sqrt{ \maxOmegai} < m$,
it is sufficient to get the bound for the term \eqref{eq:key_w} which is independent of $d$.


From now on we just write as $i,j$ for $i_k,j_k$ if it is clear from context.
The key tool we use is the following bound on $w^t_{k-1,j}$.
Since $\wvec$ is updated by incremental gradient descent via 
\begin{align}
  w^t_{k,j} = w^t_{k-1,j} - \eta_t \rbr{\frac{\lambda}{\abr{\Omegabar_j}} w^t_{k-1,j} - \frac{1}{m} \alpha^t_{k-1,i} x_{ij}} = \rbr{1 - \frac{\lambda}{\abr{\Omegabar_j}} \eta_t} w^t_{k-1,j} + \frac{1}{m} \eta_t \alpha^t_{k-1,i} x_{ij}.
\end{align}
Using $\lambda \eta_t < 1$, we derive
\begin{align}
  &\abr{w^t_{k,j} - w^t_j} \le \rbr{1-\rbr{1-\frac{\lambda}{\abr{\Omegabar_j}} \eta_t}^{\abr{\Omegabar_j}}} \abr{w^t_j} + c \eta_t v_j, \\
  \text{where} \quad  &v_j := \frac{1}{m} \sum_{i =1}^m \abr{x_{ij}}, \quad \vvec := (v_1, \ldots, v_d)'.
\end{align}
The geometric term can be bounded by
\begin{align}
  1-\rbr{1-\frac{\lambda}{\abr{\Omegabar_j}} \eta_t}^{\abr{\Omegabar_j}} = \lambda \eta_t - \sum_{k=2}^{\abr{\Omegabar_j}} {{\abr{\Omegabar_j}}\choose{k}} \rbr{\frac{-\lambda}{\abr{\Omegabar_j}}\eta_t}^k \le \lambda \eta_t + (\lambda \eta_t)^2 \sum_{k=2}^{\abr{\Omegabar_j}} \frac{1}{k!} \le c \lambda \eta_t.
\end{align}
Therefore we conclude
\begin{align}
\label{eq:bound_wt}
  \abr{w^t_{k,j} - w^t_j} \le c \eta_t (\abr{w^t_j} + v_j), \qquad \text{and} \qquad \abr{w^t_{k,j}} \le c\abr{w^t_j} + c \eta_t v_j.
\end{align}

Now we can get
\begin{align}
    \sum_k \nbr{\nabla_{\wb} f_{k}\rbr{\wb_{k-1}^{t}, \alphab_{k-1}^{t}}}^{2} 
    &\le \sum_k 2\abr{\frac{\alpha_i}{m}x_{ij}}^2 + 2\abr{\frac{\lambda}{\abr{\Omegabar_j}} w^t_{k-1,j}}^2\\
    \label{eq:bound_w}
    &\le \frac{2c^3}{m^2} \sum_k x^2_{ij} + \sum_{(i,j) \in \Omega} \frac{2\lambda^2}{\abs{\Omegabar_j}^2} \rbr{c\abr{w^t_{j}} + c \eta_t v_j}^2 \\
    &\le \frac{2c^3}{m} + \sum_{j } \frac{2c^2}{\abs{\Omegabar_j}} \rbr{c^2 (w^t_j)^2 + c^2 \eta_t^2 v_j^2} \\
    &\le \frac{2c^3}{m} + 2c^2 ( c^2 \nbr{\wvec^t}^2 + c^2 \eta_t^2  \nbr{\vvec}^2) \le c_1,    
  \end{align}
with a constant $c_1$.
Also for the rest of the terms in \eqref{eq:key_w}, we can see
\begin{align}
  f_{k}\rbr{\wb^t,\alphab} - f_{k}\rbr{\wb^t_{k-1}, \alphab} 
  &=
  \frac{\lambda}{2 \abr{\Omegabar_j}} (w^t_{j})^2 - \frac{\lambda}{2\abr{\Omegabar_j}} (w^t_{k-1,j})^2 -\frac{\alpha_i}{m} w^t_{j} x_{ij} + \frac{\alpha_i}{m} w^t_{k-1,j} x_{ij} \\
  &\le \abr{w^t_{k-1,j} - w^t_j} \rbr{\frac{2c^2}{m} \abr{x_{ij}} + \frac{\lambda}{2\abr{\Omegabar_j}} \abr{w^t_{k-1,j} + w^t_j}} \\
  &\le {c \eta_t} (\abr{w^t_j} + v_j) \rbr{\frac{2c^2}{m} \abr{x_{ij}} + \frac{\lambda}{\abr{\Omegabar_j}}(1+c) \abr{w^t_j} + \frac{\lambda c}{\abr{\Omegabar_j}} \eta_t v_j} \quad \text{by }\eqref{eq:bound_wt}.
\end{align}

Therefore
\begin{align}
  \sum_k f_{k}\rbr{\wb^t,\alphab} - f_{k}\rbr{\wb^t_{k-1}, \alphab} 
  &\le {2c^3 \eta_t} \sum_{(i,j) \in \Omega} (\abr{w^t_j} + v_j) \rbr{\frac{1}{m} \abr{x_{ij}} + \frac{1}{\abr{\Omegabar_j}} \abr{w^t_j} + \frac{1}{\abr{\Omegabar_j}} \eta_t v_j} \\
  &\le {2c^3 \eta_t} \sum_{j} (\abr{w^t_j} + v_j) \rbr{\frac{1}{m}\sum_{i=1}^m \abr{x_{ij}} + \abr{w^t_j} + \eta_t v_j} \\
  &\le {2c^3 \eta_t} (\nbr{\wvec^t}+\nbr{\vvec}) \rbr{\frac{1}{m} \sum_{i=1}^m \nbr{\xvec_i} + \nbr{\wvec^t} + \eta_t \nbr{\vvec} } \\
  &\le c_2 \eta_t,
\end{align}
with a constant $c_2$.
In the end, we can see, instead of \eqref{eq:factor_eta_t_sq}, we get
\begin{align}
    &\nbr{\wb^{t+1} - \wb}^{2} + \nbr{\alphab^{t+1} - \alphab}^{2} -
     \nbr{\wb^{t} - \wb}^{2}  -  \nbr{\alphab^{t} - \alphab}^{2}    
    \le -2 \eta_t \rbr{ f\rbr{\wb^{t}, \alphab} - f\rbr{\wb, \alphab^{t}} } \nonumber \\
    \label{eq:factor_eta_t_sq2}
    & + 2\eta_t \cdot c_2 \eta_t + c_1\eta_t^2
      + \frac{c^4\eta_t^2\sqrt{\maxOmegai}}{m}
      + \frac{4c^4}{m} \cdot \eta_t^2.
\end{align}
Thus we can get such a sufficiently large constant $C$
that $C \eta_t^2$ upper bounds the term of \eqref{eq:factor_eta_t_sq2}.
\end{proof}

\end{document}